\documentclass[sigconf,natbib=true]{acmart}
\usepackage{subfigure}
\usepackage{graphicx}
\usepackage{url}
\usepackage{balance}
\usepackage{booktabs}
\usepackage{makecell}
\definecolor{darkblue}{rgb}{0,0.08,0.45}
\usepackage{hyperref}
\usepackage{proof}
\usepackage{amsmath,amsfonts}
\newtheorem{theorem}{Theorem}
\usepackage{setspace}
\usepackage{amsmath}
\usepackage{pifont}
\usepackage{adjustbox}
\usepackage{makecell}
\usepackage{mathtools}
\usepackage{multirow}
\usepackage{multicol}
\usepackage{bm}

\usepackage{xcolor,colortbl}
\definecolor{LightCyan}{rgb}{0.88,1,1}

\newcommand\wu[1]{\textcolor{black}{#1}}
\newcommand\john[1]{\textcolor{black}{#1}}
\newcommand\wbq[1]{\textcolor{black}{#1}}

\AtBeginDocument{%
  }

\copyrightyear{2023}
\acmYear{2023}
\setcopyright{acmlicensed}\acmConference[CIKM '23]{Proceedings of the 32nd
ACM International Conference on Information and Knowledge
Management}{October 21--25, 2023}{Birmingham, United Kingdom}
\acmBooktitle{Proceedings of the 32nd ACM International Conference on
Information and Knowledge Management (CIKM '23), October 21--25, 2023,
Birmingham, United Kingdom}
\acmPrice{15.00}
\acmDOI{10.1145/3583780.3614868}
\acmISBN{979-8-4007-0124-5/23/10}

\begin{document}

\title{Enhancing the Robustness via Adversarial Learning and Joint Spatial-Temporal Embeddings in Traffic Forecasting\\
}



\author{Juyong Jiang}
\authornote{Equal contribution.}
\affiliation{%
  \institution{The Hong Kong University of Science and Technology (Guangzhou)}
  \streetaddress{}
  \city{}
  \country{}}
\email{csjuyongjiang@gmail.com}
\author{Binqing Wu}
\authornotemark[1]
\affiliation{%
  \institution{Zhejiang University}
  \streetaddress{}
  \city{}
  \country{}}
\email{binqingwu@cs.zju.edu.cn}

\author{Ling Chen}
\authornote{Corresponding authors.}
\affiliation{%
  \institution{Zhejiang University}
  \streetaddress{}
  \city{}
  \country{}}
\email{lingchen@cs.zju.edu.cn}

\author{Kai Zhang}
\authornotemark[2]
\affiliation{%
  \institution{East China Normal University}
  \streetaddress{}
  \city{}
  \country{}}
\email{kzhang@cs.ecnu.edu.cn}

\author{Sunghun Kim}
\affiliation{%
  \institution{The Hong Kong University of Science and Technology (Guangzhou)}
  \streetaddress{}
  \city{}
  \country{}}
\email{hunkim@ust.hk}
\renewcommand{\shortauthors}{Jiang and Wu, et al.}

\begin{abstract}
Traffic forecasting is an essential problem in urban planning and computing. The complex dynamic spatial-temporal dependencies among traffic objects (e.g., sensors and road segments) have been calling for highly flexible models; unfortunately, sophisticated models may suffer from poor robustness especially in capturing the trend of the time series (1st-order derivatives with time), leading to unrealistic forecasts. To address the challenge of balancing dynamics and robustness, we propose TrendGCN, a new scheme that extends the flexibility of GCNs and the distribution-preserving capacity of generative and adversarial loss for handling sequential data with inherent statistical correlations. On the one hand, our model simultaneously incorporates spatial (node-wise) embeddings and temporal (time-wise) embeddings to account for heterogeneous space-and-time convolutions; on the other hand, it uses GAN structure to systematically evaluate statistical consistencies between the real and the predicted time series in terms of both the temporal trending and the complex spatial-temporal dependencies. Compared with traditional approaches that handle step-wise predictive errors independently, our approach can produce more realistic and robust forecasts. Experiments on six benchmark traffic forecasting datasets and theoretical analysis both demonstrate the superiority and the state-of-the-art performance of TrendGCN. Source code is available at \href{https://github.com/juyongjiang/TrendGCN}{\color{blue}{https://github.com/juyongjiang/TrendGCN}}.
\end{abstract}




\maketitle

\section{Introduction}
\label{introduction}
Traffic forecasting, as one of the essential parts of the intelligent transportation system, plays an irreplaceable role in developing a smart city \cite{traffi_gcn_survey,dl_traffic_survey}. It aims to accurately predict future traffic data, e.g., traffic flow and speed, given historical traffic data recorded by sensors on a road network \cite{dcrnn}. It is a highly challenging task due to dynamic spatial and temporal dependencies within the road network. As shown in Fig. \ref{fig:dynamic}, spatially, the traffic conditions of nearby sensors have dynamic dependencies on each other. Temporally, current traffic data are dependent on historical observations \wu{in a dynamic way}. Spatial and temporal dependencies vary with time due to various factors, e.g., weather and traffic accidents. 
\begin{figure}[t]
	\centering
        \includegraphics[width=0.9 \linewidth]{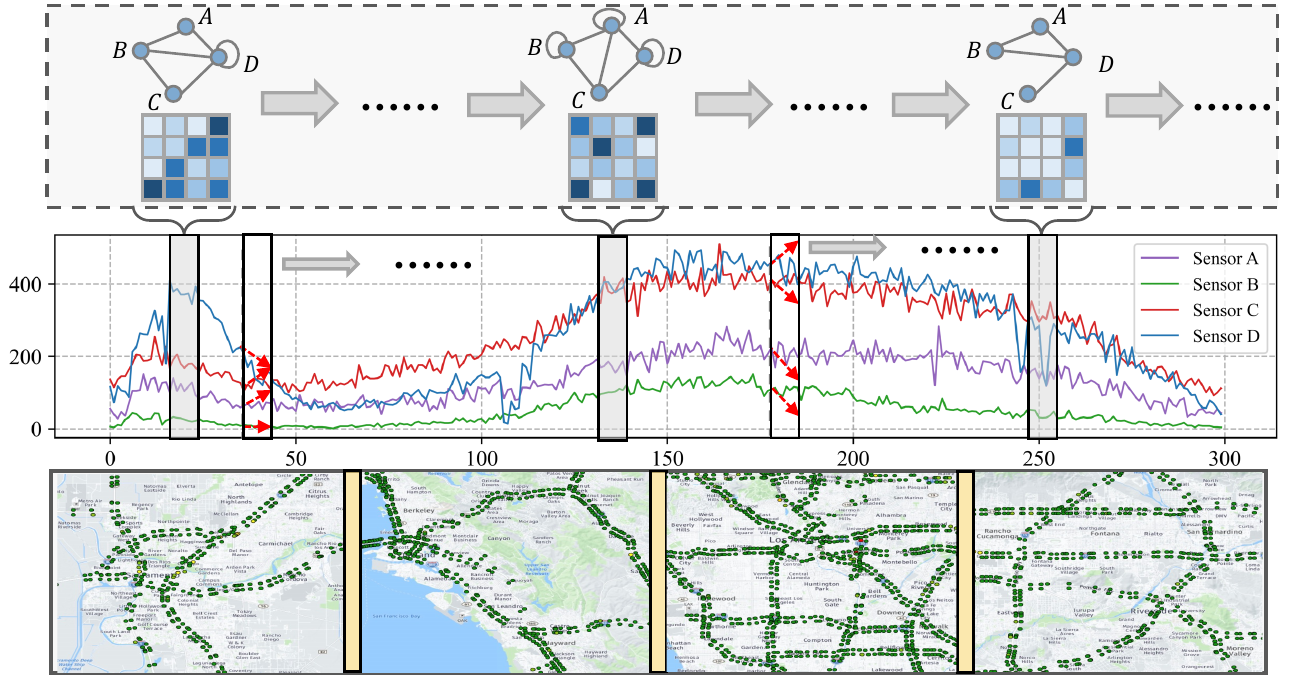}
        \caption{\john{Traffic flow data observed from 4 sensors from the PEMS04 dataset (sensor B, C, and D are adjacent to each other, and sensor A is distant from all of them). Top row: dynamic spatial-temporal dependencies among the sensors; middle row: raw time series signal from the sensors and red arrows signifying the trend (derivative); bottom row: geographical locations of the sensors. 
        }}
	\label{fig:dynamic}
\end{figure}
Many approaches have been proposed for traffic forecasting, continuously improving from shallow machine learning \cite{arima,var,svr} to recurrent neural network (RNN) and convolutional neural network (CNN) based deep learning \cite{fc_lstm,cnn-rnn,cnn-rnn-2}. Although these works can capture temporal dependencies and regular spatial dependencies, they can not adequately model non-Euclidean spatial dependencies dominated by irregular road networks. Towards this problem, graph neural networks (GNN) \cite{gnn} have been introduced in traffic forecasting owing to their superior ability to deal with irregular graph-structured data. These GNN-based works normally represent sensors as nodes and spatial dependencies between sensors as edges and leverage adjacency matrices to describe spatial dependencies of road networks \cite{gnn_survey,traffi_gcn_survey}. Recently, spatial-temporal graph neural networks (STGNNs) \cite{dcrnn,stgcn,astgcn,gman,astgnn,dstagnn}, a group of approaches integrating GNNs to model spatial dependencies with RNNs, CNNs, or Attentions to model temporal dependencies, have shown the state-of-the-art performance for traffic forecasting. 

Despite the success, there are still some limitations with current STGNNs, which we discuss below.

Firstly, most existing STGNNs rely on a basic assumption that spatial dependencies are fixed over time. Therefore, static graphs, e.g., distance graphs \cite{stgcn,astgcn,astgnn}, temporal similarity graphs \cite{stfgcn,stgode}, static adaptive graphs \cite{agcrn,mtgnn}, and their combinations \cite{graph_wavenet,stmgcn,tfgan}, are typically used to model spatial dependencies. These works do not cater to the changing nature of dependencies between nodes (shown in Fig. \ref{fig:dynamic}(a)) and cannot handle dynamic spatial dependencies. Some attempts \cite{slcnn,dgcrn,dstagnn} have tried to model such dynamics for traffic forecasting. \wu{They design feature extraction mechanisms to quantify changing patterns from the data, and with the help of domain knowledge (e.g., road occupancy rates and weather conditions) to construct time-varying spatial graphs. Compared to those based on static graphs, these works can make more realistic predictions. However, when there exist outlier points or interrupts, they could generate bad predictions, due to the sensitivity to the temporal changes (see Fig. 2(b)). Such a phenomenon calls for effective constraints on  global properties for robust time series forecasting.
}

\wu{Intuitively, since the trend of traffic data represents the average traffic conditions over time, we take the trend as a representative global property of time series. However, most existing STGNNs \cite{gnn_survey,traffi_gcn_survey,dl_traffic_survey} adopt the mean absolute error (MAE) as a loss function to evaluate the predictions and supervise the model training, which treats each predicted result individually and can not take the trends for global constraints.
As illustrated in Fig. \ref{fig:curvature}(a), the blue and the pink curves have the same magnitude $\mathcal{L}^{P_1}_{MAE}=\mathcal{L}^{P_2}_{MAE}$. The blue curve looks less desirable than the pink one when a sudden change happens around $t = 5$, as its trend is opposite to that of the ground truth, while the pink curve is consistent with the ground truth. Therefore, we should introduce more reliable constraints on trends. In particular, we term the phenomenon that predictions have different trends with the same loss values as \texttt{trend discrepancy}.} 

\wu{Recently, a few works \cite{ast,tfgan} have been proposed to eliminate trend discrepancies via GAN. They construct the true and fake samples for discriminators by concatenating inputs with predictions (from the generator) and ground truth (from the dataset), respectively. Since these works take the whole sequence in error evaluation, they can eliminate trend discrepancies and error accumulation to some extent. However, the dynamic spatial dependencies in the generator are not fully taken into account, which are crucial to capturing the changing nature of traffic systems. Moreover, spatial dependencies in the predicted results are not modelled explicitly. Since spatial dependencies reflect the hidden correlations between the trends of traffic data, they should also align with the dependencies in the ground truth.} 

\begin{figure}[t]
	\centering
        \includegraphics[width=\linewidth]{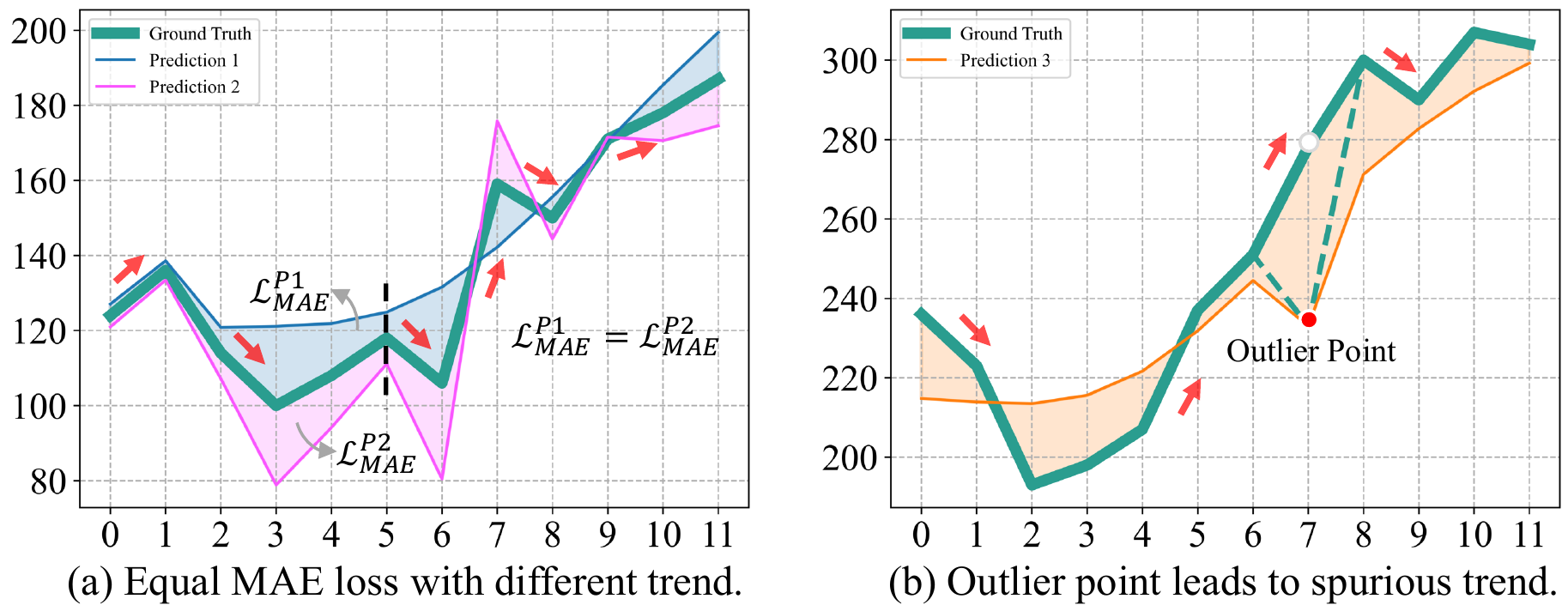}
        \caption{A motivating example with thick green curve being the raw time series and red arrows signifying the trend (derivative). {(a): prediction 1 (blue curve) and prediction 2 (pink) have the same MAE, but prediction 1 is obviously more realistic; (b): penalizing mean approximation error of the derivative of the time series can be very sensitive to outliers in the signal and lead to undesired prediction.}} 
	\label{fig:curvature}
\end{figure}
\wu{To this end, we propose TrendGCN to solve the two aforementioned problems: 1) how to model dynamic spatial dependencies concisely and effectively; 2) how to coordinate the trend discrepancies with dynamic modeling to improve the robustness}. The main contributions of our work are summarized as follows:


\begin{itemize}
    \item \john{We propose TrendGCN, a new scheme combing the flexibility of GCNs and the capacity of generative and adversarial loss in sequential data with inherent statistical correlations. It employs simultaneous spatial (node-wise) embedding and temporal (time-wise) embedding to account for heterogeneous space-and-time convolutions.} 
    \item \john{We introduce adversarial training to systematically evaluate both the trend-level and dependency-level discrepancies between the true data and the predicted results, thus being more robust in generating a desired trend than handling step-wise prediction errors independently.}
    \item \john{We evaluate the proposed model on six benchmarks traffic forecasting datasets. Extensive experiments and theoretical analysis both demonstrate the superiority and the state-of-the-art performance of TrendGCN.}
\end{itemize}

\section{Related Work}
\subsection{STGNNs for Traffic Forecasting}

Spatial-temporal graph neural networks (STGNNs) \cite{dcrnn,stgcn,stmgcn,astgcn,gman,astgnn} have shown remarkable performance and achieved state-of-the-art in traffic forecasting. They mainly integrate GNNs to model non-Euclidean spatial dependencies with RNNs, CNNs, and Attentions to model temporal dependencies \cite{gnn_survey,traffi_gcn_survey}. However, many existing STGNNs utilize static adjacency matrices, which neglect the changing nature of spatial dependencies in road networks.

Some recent STGNNs \cite{stemgnn,dstagnn,dgcrn,slcnn} are designed to model dynamic spatial dependencies. For example, DGCNN \cite{dgcnn} decomposes the static and dynamic components of traffic data based on a pre-trained tensor decomposition layer to obtain the dynamic Laplacian matrix at any time. SLCNN \cite{slcnn} proposes global and local time-varying structure learning convolutional modules. Each module encodes the static structure by a learnable matrix, and the dynamic structure by a function \john{taking the current samples as inputs.} DCGRN \cite{dgcrn} adopts dynamic adjacency matrices by integrating dynamic context features, e.g., the speed and the time of day. DSTAGNN \cite{dstagnn} obtains the dynamic adjacency matrix according to a cosine distance based distance adjacency matrix and an improved self-attention. 
\wu{However, these works usually rely on complex mechanisms to capture dynamic dependencies, which may introduce too many parameters and face the high risk of over-fitting. In addition, some of them depend on domain dynamic factors (e.g., road occupancy rates and weather conditions) heavily, losing the robustness and generalization of models for different applications to some extent. Therefore, how to design an architecture to model dynamic spatial dependencies concisely yet effectively is an open problem for both academic and industrial communities.}

\subsection{GANs for Times Series}
Generative Adversarial Networks (GANs) can learn to produce realistic data adversarially. They have achieved remarkable success in computer vision \cite{gan_cv} and natural language processing \cite{gan_nlp}, and have also shown promise in time series analysis. TimeGAN \cite{timegan} first introduces GANs to time series generation. It utilizes GANs based on a learned embedding space to generate time series that preserves temporal dynamics. AST \cite{ast} promotes GANs for time series forecasting. It adopts a sparse transformer as the generator to learn a sparse attention map and uses a discriminator to eliminate the error accumulation at the sequence level. TrafficGAN \cite{trafficgan} utilizes GANs for traffic forecasting. It applies CNN and LSTM to capture the spatial-temporal dependencies, with adversarial training to learn the distribution of future traffic flows. 
More recently, TFGAN \cite{tfgan} integrates GAN and GCNs for traffic forecasting, which uses GAN to learn the distribution of the time series data. Specifically, multiple static graphs are constructed within the generator to model spatial dependencies. The discriminator constructs the true and fake samples at the sequence level by concatenating inputs with predictions and ground truth, respectively. 

These models typically use GANs for learning the distribution of time series data from a static perspective, but not fully catering to dynamic spatial dependencies in the generative or discrimination process.  In addition, these methods barely explicitly consider the global properties of traffic data, e.g., the overall trend of each time series and the correlations between different sensors (or channels), which are critical for traffic forecasting.

\section{Methodology}
\subsection{Problem Definition}
In this paper, we aim to solve multi-step traffic forecasting problems, given the observed historical time series. Formally, we define these time series as a set $\boldsymbol{X}^{1:T} = \{\boldsymbol{X}^{(1)}, \boldsymbol{X}^{(2)}, \cdots \boldsymbol{X}^{(t)}, \cdots \boldsymbol{X}^{(T)}\} \in \mathbb{R}^{T\times N \times F}$, where $\boldsymbol{X}^{(t)} \in \mathbb{R}^{N \times F}$ denotes observed values with $F$ feature dimensions of $N$ nodes at time step $t$, and $\boldsymbol{X}_i^{(t)}$ represents the value of the $i$-th node at time step $t$. Our target is to find a mapping function $\mathcal{F}$ to forecast the next $H$ steps data based on the past $T$ steps data. Thus, the traffic forecasting problem can be formulated as follows: 
\begin{equation}
\begin{aligned}
    \hat{\boldsymbol{X}}^{T+1:T+H} = \mathcal{F}(\boldsymbol{X}^{1:T}; \boldsymbol{\Theta})
\end{aligned}
\end{equation}
where $\hat{\boldsymbol{X}}^{T+1:T+H} \in \mathbb{R}^{H \times N \times O}$, $H$ denotes the forecasting horizon and $O$ is the output feature dimensions of each node. $\mathcal{F}$ is the mapping function, and  $\boldsymbol{\Theta}$ denotes all learnable parameters in the model.
\begin{figure*}[t]
	\centering
	\includegraphics[width=0.7\linewidth]{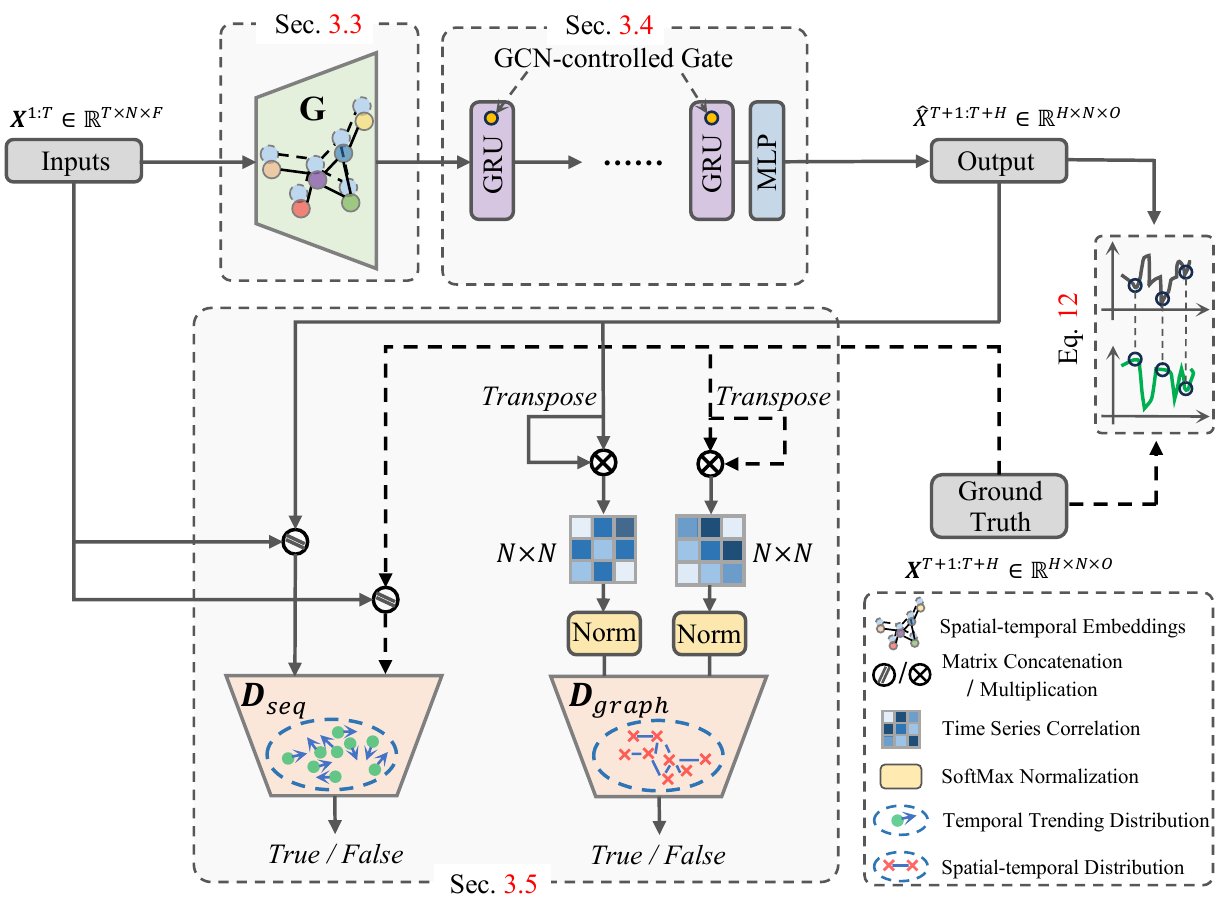}
	\caption{The model architecture of the proposed TrendGCN. The Detailed description of each proposed component can be found in the corresponding section (marked by the red digit). }
	\label{fig:daagcn}
\end{figure*}
\subsection{Model Overview}
\wu{Fig. \ref{fig:daagcn} shows the architecture of TrendGCN that mainly consists of a generator with dynamic adaptive graph generation for capturing dynamic spatial dependencies and two discriminators for evaluating and trying to eliminate the trend-level and dependency-level discrepancies}

\subsection{Dynamic Adaptive Graph Generation}
\wu{Recently, adaptive graph generation methods have been prevalent for traffic forecasting, as they can learn spatial dependencies from data automatically and help to find some hidden patterns. Particularly, some works \cite{graph_wavenet,agcrn,magnn} learn graphs in a simple way. They parameterize the representations of all nodes directly using learnable node-wise embeddings, calculate the pairwise similarity of these representations, and treat this similarity matrix as the adjacency matrix of nodes. However, these works can only obtain static graphs and can not model the changing spatial dependencies among nodes. Therefore, we propose a \underline{D}ynamic \underline{A}daptive \underline{G}raph \underline{G}eneration module to model dynamic spatial dependencies concisely yet effectively in an adaptive fashion.}

\wu{Inspired by the positional embeddings of Transformers \cite{dai2019transformer,jiang2020cascaded}, we utilize two types of embeddings, spatial embeddings $\bm{E}_{\text{node}}=\{\bm{e}_{\text{node}}^{(1)},\bm{e}_{\text{node}}^{(2)}, \cdots,\bm{e}_{\text{node}}^{(N)}\} \in R^{N \times d_{\text{e}}}$ and temporal embeddings $\bm{E}_{\text{time}}=\{\bm{e}_{\text{time}}^{(1)},\bm{e}_{\text{time}}^{(2)}, \cdots,\bm{e}_{\text{time}}^{(T)}\} \in R^{T \times d_{\text{e}}}$ to denote the unique representations of each node and each time step, respectively. In detail, the $i$th row of $\bm{E}_{\text{node}}$ denotes the representations of the $i$th node, the $i$th row of $\bm{E}_{\text{time}}$ denotes the representations of the $i$th time step, and $d_{\text{e}}$ is the hidden dimension of spatial and temporal embeddings.} 

\wu{We introduce a unified scheme to effectively couple the spatial (node-wise) and temporal (time-wise) embeddings through a gate module and use the integrated embeddings to construct graphs changing over time. The process can be formulated as:}
\begin{equation}
\label{eq:node_time_1}
\begin{aligned}
\bm{\mathcal{A}}_{ij}^{(t)} = \bm{\lambda}\left \langle \text{Dpt}\left(\text{LN}(\bm{e}_{{\text{node}}}^{(i)} \Delta_1 \bm{e}_{{\text{time}}}^{(t)})\right), \text{Dpt}\left(\text{LN}(\bm{e}_{{\text{node}}}^{(j)} \Delta_2 \bm{e}_{{\text{time}}}^{(t)})\right)\right \rangle
\end{aligned}
\end{equation}
\wu{where $\Delta_1, \Delta_2$ denote two operators selected from a set of candidate operators: addition, Hadamard production, and concatenation, abbreviated as $\{+, \odot, \mathbin{\|}\}$; the $\text{LN}$ and $\text{Dpt}$ denote Layer Normalization and Dropout operation, respectively. $\left \langle\cdot, \cdot\right \rangle$ denotes the inner product, and $\bm{\lambda}$ represents the important weights of each kind of information term. The choices of $\Delta_1, \Delta_2$ can be the same or different, and the corresponding experiment results and analysis about their combinations are in the Appendix. In particular, when $\Delta_1=+, \Delta_2=+$, Eq. \ref{eq:node_time_1} can be expanded as:}
\begin{equation}
\label{eq:node_time_expand}
\begin{aligned}
\bm{\mathcal{A}}_{ij}^{(t)} &= \bm{\lambda}\left \langle\text{Dpt}\left(\text{LN}(\bm{e}_{{\text{node}}}^{(i)} + \bm{e}_{\text{time}}^{(t)})\right), \text{Dpt}\left(\text{LN}(\bm{e}_{{\text{node}}}^{(j)} + \bm{e}_{\text{time}}^{(t)})\right)\right \rangle \\
&= \lambda_1\underbrace{\left \langle\bm{e}_{{\text{node}}}^{(i)}, \bm{e}_{{\text{node}}}^{(j)}\right \rangle}_{\text{\john{spatial} homologous terms}} + \lambda_2\underbrace{\left \langle\bm{e}_{{\text{node}}}^{(i)}, \bm{e}_{\text{time}}^{(t)}\right \rangle + \left \langle\bm{e}_{{\text{node}}}^{(j)}, \bm{e}_{\text{time}}^{(t)}\right \rangle}_{\text{\john{spatial-temporal} heterologous terms}} \\
&+ \lambda_3\underbrace{\left \langle\bm{e}_{\text{time}}^{(t)}, \bm{e}_{\text{time}}^{(t)}\right \rangle}_{\text{\john{temporal} homologous terms}} 
\end{aligned}
\end{equation}
\wu{This formulation allows not only homogeneous interactions in the spatial and temporal domains, respectively, but also allows the embedding of the $i$th node and the $j$th time step to interact directly with each other. Thus, the construed graph can represent the spatial, temporal, and spatial-temporal interactions simultaneously, which has a stronger representative ability than a static adaptive graph that only focuses on spatial interactions. In particular, a static adaptive graph is a special case of our graph when $\lambda_2$ and $\lambda_3$ are equal to zero.}

Finally, following previous works \cite{agcrn,rgsl}, we employ $1^{st}$ order Chebyshev polynomial expansion to approximate graph convolution with parameters \wu{that are specific to the combinations of spatial and temporal embeddings $\bm{E}_{\text{nt}}$}, then the graph convolution can be formulated as: 
\begin{equation}
\begin{aligned}
    \bm{H_{\bm{\mathcal{G}}}}^{(l+1)} = (\bm{I_N} + \text{Norm}(\bm{\mathcal{A}}^{(t)})) \bm{H_{\bm{\mathcal{G}}}}^{(l)} \bm{E}_{\text{nt}} \bm{W_\mathcal{G}}^{(l)} + \bm{E}_{\text{nt}} \bm{\mathbf{b}_\mathcal{G}}^{(l)}
\end{aligned}
\end{equation}
\begin{equation}
\begin{aligned}
    \bm{H_{\bm{\mathcal{G}}}^{(0)}} = \bm{X}^{(t)}, \ \bm{E}_{\text{nt}}= \bm{E}_{{\text{node}}} \Delta_1 \bm{E}_{{\text{time}}}^{(t)}
\end{aligned}
\end{equation}
where $\bm{I_N}$ is the identity connection of $N$ nodes, $\text{Norm}$ is $Softmax$ normalization; $\bm{W_\mathcal{G}}^{(l)} \in R^{d \times F \times O}$ and $\bm{b_\mathcal{G}}^{(l)} \in R^{d \times O}$ represents a weight pool and a bias pool, respectively. \wu{During training, $\bm{E}_{{\text{node}}}$ and $\bm{E}_{\text{time}}$ are updated. Thus, the constructed graphs are dynamics, and the parameters of the graph convolution operation $\bm{E}_{\text{nt}} \bm{W_\mathcal{G}}^{(l)}$ and $\bm{E}_{\text{nt}} \bm{\mathbf{b}_\mathcal{G}}^{(l)}$ are specific to nodes and time steps.}


\subsection{Dynamic Graph Convolutional GRU} 
Following prior works \cite{agcrn,rgsl}, we integrate the proposed DAGG module to Gated Recurrent Units (GRU) \cite{chung2014empirical} by replacing the MLP layers in GRU. Then, we stack several GRU layers followed by a linear transformation (MLP) to project the $T$-th output of GRU to \wu{achieve $H$ steps ahead predictions in the manner of sequence to sequence, which significantly decreases the cost of time and error accumulation.} 
Formally, it can be formulated as:
\begin{equation}
\begin{aligned}
&\bm{z}^{(t)} =\sigma(\bm{\mathcal{G}}((\bm{X}^{(t)} \mathbin{\|} \bm{h}^{(t-1)});\bm{\Theta}_z)) \\
&\bm{r}^{(t)} =\sigma(\bm{\mathcal{G}}((\bm{X}^{(t)} \mathbin{\|} \bm{h}^{(t-1)});\bm{\Theta}_r)) \\
&\bm{c}^{t} = tanh(\bm{\mathcal{G}}((\bm{X}^{(t)} \mathbin{\|} \bm{r}^{(t)} \odot \bm{h}^{(t-1)});\bm{\Theta}_c)) \\
&\bm{h}^{(t)} =\bm{z}^{(t)} \odot \bm{h}^{(t-1)} + (1 - \bm{z}^{(t)}) \odot \bm{c}^{t}
\end{aligned}
\end{equation}
\begin{equation}
\label{eq:pred}
\begin{aligned}
&\hat{\bm{X}}^{T+1:T+H}=\text{Dpt}(\text{LN}(\bm{h}^{(T)}))\bm{W} + \bm{b}
\end{aligned}
\end{equation}  
where $\bm{X}^{(t)} \in R^{T \times N \times F} $ and $\bm{h}^{(t)} \in R^{1 \times N \times F'}$ represent input and hidden representation of GRU at time step $t$, $\mathbin{\|}$ denotes the concatenation operation, $\bm{z}^{(t)}$ and $\bm{r}^{(t)}$ denote reset gate and update gate at time step $t$, respectively. Three $\bm{\mathcal{G}}$ represents DAGG module with different learnable parameters $\bm{\Theta}_z$, $\bm{\Theta}_r$, and $\bm{\Theta}_c$. $\bm{W} \in R^{F' \times HO}$ and $\bm{b} \in R^{1 \times HO}$ are weight parameters in linear transformation (MLP). $H$ denotes the predicted future steps and $\hat{\bm{X}}^{T+1:T+H} \in R^{H \times N \times O}$ is the final prediction results. 

\subsection{Adversarial Dynamic Trend Alignment}
\label{adv_training}
\wu{We introduce two discriminators with adversarial training to take the global properties (trends and inherent statistical correlations) into consideration, which systematically evaluate trend-level and dependency-level discrepancies and further improve the robustness. Specifically, the discriminator $\bm{\mathcal{D}}_{\text{seq}}$ focuses on the trend of individual time series, and the discriminator $\bm{\mathcal{D}}_{\text{graph}}$ emphasizes the correlation of multivariate time series. Both discriminators consist of three fully connected linear layers \cite{ast} with $LeakReLU$. Formally, the loss functions of this min-max optimization problem are formulated as:}

\begin{equation}
\begin{aligned}
\label{eq:adv_loss_seq}
\mathcal{L}_{\bm{\mathcal{D}}_{\text{seq}}}=- &\mathbb{E}_{x_r^1 \sim \mathbb{P}} [ \log ( \bm{\mathcal{D}}_{\text{seq}} ( \bm{X}^{1:T} \mathbin{\|} \bm{X}^{T+1:T+H} ) )]\\
-\mathbb{E}_{x_{f}^1 \sim \mathbb{Q}}& [\log ( 1 - \bm{\mathcal{D}}_{\text{seq}} ( \bm{X}^{1:T} \mathbin{\|} \hat{\bm{X}}^{T+1:T+H}))]
\end{aligned}
\end{equation}
\begin{equation}
\begin{aligned}
\label{eq:adv_loss_graph}
\mathcal{L}_{\bm{\mathcal{D}}_{\text{graph}}}=- &\mathbb{E}_{x_r^2 \sim \mathbb{P}} [\log ( \bm{\mathcal{D}}_{\text{graph}}( \delta((\bm{X}^{T+1:T+H})^{\mathcal{T}}\bm{X}^{T+1:T+H}))]\\
-\mathbb{E}_{x_{f}^2 \sim \mathbb{Q}} [\log & ( 1 - \bm{\mathcal{D}}_{\text{graph}}  \delta((\hat{\bm{X}}^{T+1:T+H})^{\mathcal{T}}\hat{\bm{X}}^{T+1:T+H}) ) ) ] 
\end{aligned}
\end{equation}
\begin{equation}
\begin{aligned}
\label{eq:adv_loss_g}
\mathcal{L}_{\text{adv}} = \alpha&( -\mathbb{E}_{x_r^1 \sim \mathbb{P}} [ \log (1 - \bm{\mathcal{D}}_{\text{seq}} ( \bm{X}^{1:T} \mathbin{\|} \bm{X}^{T+1:T+H} ) ) ]\\
&- \mathbb{E}_{x_f^1 \sim \mathbb{Q}} [ \log (\bm{\mathcal{D}}_{\text{seq}} (\bm{X}^{1:T} \mathbin{\|} \hat{\bm{X}}^{T+1:T+H} ) ) ] )\\
+\beta(-\mathbb{E}_{x_r^2 \sim \mathbb{P}}& [\log (1 - \bm{\mathcal{D}}_{\text{graph}} ( \delta((\bm{X}^{T+1:T+H})^{\mathcal{T}}\bm{X}^{T+1:T+H}) ) ) ] \\
- \mathbb{E}_{x_f^2 \sim \mathbb{Q}}& [\log (\bm{\mathcal{D}}_{\text{graph}} (\delta((\hat{\bm{X}}^{T+1:T+H})^{\mathcal{T}}\hat{\bm{X}}^{T+1:T+H}) ) ) ])
\end{aligned}
\end{equation}
Here, $\bm{x}_r^1=(\bm{X}^{1:T} \mathbin{\|} \bm{X}^{T+1:T+H})$ and $\bm{x}_r^2=\delta((\bm{X}^{T+1:T+H})^{\mathcal{T}}\bm{X}^{T+1:T+H})$ denote the ground truth (real) sampled from distribution $\mathbb{P}$, $\bm{x}_f^1=(\bm{X}^{1:T} \mathbin{\|} \hat{\bm{X}}^{T+1:T+H})$ and $\bm{x}_f^2=\delta((\hat{\bm{X}}^{T+1:T+H})^{\mathcal{T}}\hat{\bm{X}}^{T+1:T+H})$ is the predicted (fake) time series sampled from distribution $\mathbb{Q}$. $\mathcal{T}$ and $\mathbin{\|}$ denote the transpose and concatenation operations, respectively, $\delta(\cdot)$ is $softmax$ normalization operation. $\alpha$ and $\beta$ represent the trade-off weights to balance the importance of $\bm{\mathcal{D}}_{\text{seq}}$ and $\bm{\mathcal{D}}_{\text{graph}}$. 

\subsection{Multivariate Time Series Prediction} 
We utilize L1 loss as training objective and jointly optimize the loss with the adversarial training loss for the generator to make multi-step predictions. Thus, the overall loss of our TrendGCN is formulated as:
\begin{equation}
\label{loss_objective}
\begin{aligned}
\mathcal{L} &= \mathcal{L}_{p}(\bm{\Theta}) + \mathcal{L}_{\text{adv}}
\end{aligned}
\end{equation}
\begin{equation}
\label{eq:l1loss}
\begin{aligned}
\mathcal{L}_{p}(\bm{\Theta}) &= \sum_{t=T+1}^{T+H}\left\|\bm{X}^{(t)}-\bm{\hat{X}}^{(t)}\right\|
\end{aligned}
\end{equation}
where $\bm{\hat{X}}^{(t)} \in R^{N \times O}$ and $\bm{X}^{(t)} \in R^{N \times O}$ denote ground truth and predicted results of all nodes at time step $t$, $\bm{\Theta}$ is all the learnable parameters in the model. 

\section{Theoretical Analysis}
In this section, we theoretically show that models which individually and independently consider the absolute error between ground truth and predictions at different time steps will result in \emph{trend discrepancy}, \john{namely, different predictions have different trends from ground truth while having the same absolute error with ground truth (See Fig. \ref{fig:curvature}(a))}, and the \john{functionality} of introducing adversarial training. 
\begin{theorem} \label{thm:rank}
Let $\mathcal{F^*}$ denotes the optimal model 
with parameters $\boldsymbol{\Theta}$ to predict the next $H$ steps data $\hat{\boldsymbol{X}}^{T+1:T+H}=\{\hat{\boldsymbol{X}}^{(T+1)}, \hat{\boldsymbol{X}}^{(T+2)}, \cdots  \\ \hat{\boldsymbol{X}}^{(T+t)}, \cdots \boldsymbol{X}^{(T+H)}\} \in \mathbb{R}^{H \times N \times O}$, given the past $T$ steps data $\boldsymbol{X}^{1:T} = \{\boldsymbol{X}^{(1)}, \boldsymbol{X}^{(2)}, \cdots \boldsymbol{X}^{(t)}, \cdots \boldsymbol{X}^{(T)}\} \in \mathbb{R}^{T \times N \times F}$, i.e., $\hat{\boldsymbol{X}}^{T+1:T+H} = \\\mathcal{F^*}(\boldsymbol{X}^{1:T}; \boldsymbol{\Theta})$, \john{using L1 loss represents prediction errors}. Then, there always exists another mapping function $\widetilde{\mathcal{F}}$ with the same loss between ground truth and predictions at each time step, \john{but with the different derivative of the predicted time series at each time step (i.e., $\frac{\mathrm{d} \mathcal{F^*} }{\mathrm{d} t}$)} 
\end{theorem}
\begin{proof}[Proof of Theorem~\ref{thm:rank}]
According to Eq. \ref{eq:l1loss}, the L1 loss of mapping function $\mathcal{F^*}$ and $\widetilde{\mathcal{F}}$ can be formulated as:
\begin{equation}
    \begin{aligned}
    \mathcal{L}_{\mathcal{F^*}} = \sum_{t=T+1}^{T+H}\left\|\bm{X}^{(t)}-\bm{\hat{X}}^{(t)}_\mathcal{F^*}\right\|, \mathcal{L}_{\widetilde{\mathcal{F}}} = \sum_{t=T+1}^{T+H}\left\|\bm{X}^{(t)}-\bm{\hat{X}}^{(t)}_{\widetilde{\mathcal{F}}}\right\|
    \end{aligned}
\end{equation}
Obviously, for $t\in[T+1, T+H]$ we have the following inequality:
\begin{equation}
    \begin{aligned}
    min\left\{\|\bm{X}^{(t)}-\bm{\hat{X}}^{(t)}_\mathcal{F^*}\|\right\}\leq \frac{\mathcal{L}_{\mathcal{F^*}}}{H} \leq max\left\{\|\bm{X}^{(t)}-\bm{\hat{X}}^{(t)}_\mathcal{F^*}\|\right\} \\
    min\left\{\|\bm{X}^{(t)}-\bm{\hat{X}}^{(t)}_{\widetilde{\mathcal{F}}}\|\right\}\leq \frac{\mathcal{L}_{\widetilde{\mathcal{F}}}}{H} \leq max\left\{\|\bm{X}^{({t})}-\bm{\hat{X}}^{({t})}_{\widetilde{\mathcal{F}}}\|\right\}
    \end{aligned}
\end{equation}
Further, when $\forall t\in[T+1, T+H]$, $\bm{\hat{X}}^{({t})}_{\widetilde{\mathcal{F}}} = - \bm{\hat{X}}^{(t)}_\mathcal{F^*} + 2\bm{X}^{(t)}$, we have $\mathcal{L}_{\widetilde{\mathcal{F}}}=\mathcal{L}_{\mathcal{F^*}}$. Then, recall the definition of \john{derivative}, we obtain:
\begin{equation}
    \begin{aligned}
    m_1^{(t)} &= \lim_{\Delta t\to0}\frac{\bm{\hat{X}}^{({t}+\Delta t)}_{{\mathcal{F^*}}}-\bm{\hat{X}}^{({t})}_{{\mathcal{F^*}}}}{\Delta t}\\
    m_2^{(t)} &= \lim_{\Delta t\to0}\frac{\bm{\hat{X}}^{({t}+\Delta t)}_{\widetilde{\mathcal{F}}}-\bm{\hat{X}}^{({t})}_{\widetilde{\mathcal{F}}}}{\Delta t} \\
    &= \lim_{\Delta t\to 0}\frac{(- \bm{\hat{X}}^{(t+\Delta t)}_\mathcal{F^*} + 2\bm{X}^{(t+\Delta t)}) - (- \bm{\hat{X}}^{(t)}_\mathcal{F^*} + 2\bm{X}^{(t)})}{\Delta t}\\ 
    &= -m_1^{(t)} + 2\lim_{\Delta t\to 0}\frac{\bm{X}^{(t+\Delta t)}-\bm{X}^{(t)}}{\Delta t}\\
    &= -m_1^{(t)} + 2m^{(t)}
    \end{aligned}
\end{equation}
Here, we use $m^{(t)}=\lim_{\Delta t\to 0}\frac{\bm{X}^{(t+\Delta t)}-\bm{X}^{(t)}}{\Delta t}$ to denote the derivative of ground truth mapping function at $t$ time step. Obviously, $\exists t\in [T+1, T+H]$, $m_1^{(t)} \neq m^{(t)}$ to have $m_2^{(t)}\neq m_1^{(t)}$. It indicates that equal approximation error $\mathcal{L}_{\widetilde{\mathcal{F}}}=\mathcal{L}_{\mathcal{F^*}}$ does not guarantee  equal trend of the predicted time series, i.e., $m_2^{(t)}\neq m_1^{(t)}$. 

Moreover, if we explicitly minimize the trend loss between prediction and ground truth at each time step, formalized by 
\begin{equation}
\begin{aligned}
\mathcal{L}_{trend}(\bm{\Theta}) = \sum_{t=T+1}^{T+H}\left\|\bm{m}^{(t)}-\bm{\hat{m}}^{(t)}\right\| 
\end{aligned}
\end{equation}
it is still sensitive to outlier values which leads to a spurious trend, as shown in Fig. \ref{fig:curvature}(b). To solve the above problems, we introduce adversarial training to discriminate whether predictions have the same trend as ground truth from a higher level 
instead of constraining the trend consistency at each time step.
\end{proof}
\begin{table*}
	\caption{Statistics of the six benchmarks traffic forecasting datasets. In the row of signals, `F' represents traffic flow, `S' represents traffic speed, and `O' represents traffic occupancy rate.}
	\label{dataset_statistics}
	\resizebox{0.7\textwidth}{!}{
	\begin{tabular}{l|l|l|l|l|l|l}
        \hline	
        Dataset  & PEMS03 & PEMS04 & PEMS07 & PEMS08 & METR-LA & PeMS-BAY  \\
	\hline
	\hline
	\# of nodes & 358 & 307 & 883 & 170 & 207 & 325\\
	\# of timesteps & 26,208 & 16,992 & 28,224 & 17,856 & 34,272& 52,116\\
	\# Granularity & 5min & 5min & 5min & 5min & 5min & 5min\\
	\# Start time  & 9/1/2018 & 1/1/2018  &  5/1/2017 & 7/1/2016 & 3/1/2012& 1/1/2017\\
	\# End time    & 11/30/2018 & 2/28/2018 & 8/31/2017& 8/31/2016 & 6/30/2012 & 5/31/2017\\
	\# Missing ratio$^*$ &  0.672\% & 3.182\% & 0.452\% & 0.696\% & 8.11\%& 0.003\%\\
	\# Signals $^*$ & F & F,S,O	& F & F,S,O & S & S\\
	\hline
	\end{tabular}
	}
\end{table*}

\section{Experiments}
\subsection{Dataset} \label{dataset}
To evaluate the proposed TrendGCN, we conduct extensive experiments with six traffic forecasting benchmarks, including PEMS03/04/07/08, METR-LA, and PeMS-BAY. The datasets PEMS03/04/07/08 and the preprocessing procedure are provided by \cite{astgnn}. The datasets METR-LA/PeMS-BAY and the preprocessing procedure are provided by \cite{dcrnn}.
The dataset statistics are summarized in Table \ref{dataset_statistics}.

\subsection{Baselines}

\wbq{We compare TrendGCN with 22 baselines of three categories. The details of the baselines are as follows:
\begin{itemize}
    \item The following simple temporal models are considered:  
    ARIMA \cite{arima}, considering moving average and autoregressive components; 
    FC-LSTM \cite{fc_lstm}, using fully connected LSTMs to capture the nonlinear temporal dependencies; 
    TCN \cite{tcn}, consisting of a stack of causal convolutional layers with exponentially enlarged dilation factors for sequence modeling tasks; 
    \item The following graph-based models are included: 
    DCRNN \cite{dcrnn}, integrating diffusion convolution with sequence-to-sequence architecture; 
    STGCN \cite{stgcn}, merging graph convolution with gated temporal convolutions; 
    ASTGCN \cite{astgcn}, integrating attention mechanisms to capture dynamic spatial-temporal patterns; 
    Graph WaveNet \cite{graph_wavenet}, combining graph convolution with dilated casual convolution; 
    STG2Seq \cite{stg2seq}, using a hierarchical graph convolutional structure to capture both spatial and temporal correlations simultaneously;
    STSGCN \cite{stsgcn}, utilizing localized spatial-temporal subgraph module to model localized correlations independently; 
    AGCRN \cite{agcrn}, using adaptive adjacency matrix for graph convolution and GRU to model temporal correlations; 
    LSGCN \cite{lsgcn}, using a spatial gated block and gated linear units convolution to capture complex spatial-temporal features;
    MTGNN \cite{mtgnn}, extracting the uni-directed relations among variables through a graph learning module;
    STFGNN \cite{stfgcn}, fusing various spatial and temporal graphs to handle long sequences; 
    Z-GCNETs \cite{z_gcnnets}, integrating the new time-aware zigzag topological layer into time-conditioned GCNs;
    STGODE \cite{stgode}, capturing spatial-temporal dynamics through a tensor-based ODE;
    DCGRN \cite{dgcrn}, adopts dynamic adjacency matrices by integrating dynamic context features, e.g., the speed and the time of day.
    STG-NCDE \cite{stgbcde}, designing two NCDEs for learning the temporal and spatial dependencies;
    DSTAGNN \cite{dstagnn}, designing a new spatial-temporal attention module to exploit the dynamic spatial correlation within multi-scale neighborhoods;
    RGSL \cite{rgsl}, incorporating both explicit prior structure and implicit structure together to learn a better graph structure.
    \item The following GAN-based models are included: 
    TimeGAN \cite{timegan}, utilizing GANs based on a learned embedding space to generate time series that preserves temporal dynamics.
    AST \cite{ast}, adopting a sparse transformer as the generator to learn a sparse attention map and uses a discriminator to eliminate the error accumulation at the sequence level.
    TFGAN \cite{tfgan}, applying multiple GCNs and one GRU within the generator to model spatial and temporal dependencies, respectively.
\end{itemize}}
\begin{table*}[t]
	\centering
	\caption{\wu{Performance comparison of different baselines for traffic flow forecasting on PEMS03/04/07/08 datasets. \textbf{Bold} scores and \underline{underline} scores indicate the best and the second best, respectively. 
 Superscript
 $\{a, b, c, d, e, f, g, h\}$ denotes methods with adaptive graphs, while $*$ denotes methods with dynamic graphs. }
}
	\resizebox{\linewidth}{!}{
	\label{tab:main_result}
        \begin{tabular}{c|ccc|ccc|ccc|ccc}
	    \hline 
        \multirow{2}*{Model} & \multicolumn{3}{c|}{PEMS03} & \multicolumn{3}{c|}{PEMS04} & \multicolumn{3}{c|}{PEMS07} & \multicolumn{3}{c}{PEMS08} \\
    	\cline{2-13}
    	& MAE & RMSE & MAPE & MAE & RMSE & MAPE & MAE & RMSE & MAPE & MAE & RMSE & MAPE \\
    	\hline
    	\hline
     	ARIMA (JTE 2003) & 35.41 & 47.59 & 33.78 \%& 33.73 &48.80 &24.18\% & 38.17 & 59.27& 19.46\%& 31.09&44.32 &22.73\%\\
     	FC-LSTM (NeurIPS 2015) & 21.33 & 35.11 &23.33\%& 26.77 &40.65 &18.23\% &29.98& 45.94& 13.20\% &23.09& 35.17 &14.99\% \\
		TCN (ICLR 2018) & 19.32 &33.55 &19.93\%& 23.22 &37.26 &15.59 \%&32.72 &42.23 &14.26\%& 22.72& 35.79 &14.03\% \\
		\hline
		DCRNN (ICLR 2018) & 17.99 &30.31 & 18.34\% & 21.22 & 33.44 &14.17\%& 25.22 & 38.61 & 11.82\% & 16.82 & 26.36 & 10.92\%\\
            STGCN (IJCAI 2018) & 17.55 &30.42 & 17.34\% & 21.16 & 34.89 &13.83\%& 25.33 & 39.34 & 11.21\% & 17.50 & 27.09 & 11.29\%\\
            ASTGCN (AAAI 2019) & 17.34 &29.56 & 17.21\% & 22.93 & 35.22 &16.56\%& 24.01 & 37.87 & 10.73\% & 18.25 & 28.06 & 11.64\%\\
            $^{a}$GraphWaveNet (IJCAI 2019) & 19.12 &32.77 & 18.89\% & 24.89 & 39.66 &17.29\%& 26.39 & 41.50 & 11.97\% & 18.28 & 30.05 & 12.15\%\\
            STG2Seq (IJCAI 2019) & 19.03& 29.83&21.55\%& 25.20& 38.48& 18.77\%& 32.77& 47.16& 20.16\% &20.17 &30.71& 17.32\%\\
		STSGCN (AAAI 2020) & 17.48& 29.21& 16.78\% & 21.19& 33.65& 13.90\%& 24.26& 39.03& 10.21 \%& 17.13& 26.80& 10.96\%\\
            $^{b}$AGCRN (NeurIPS 2020) & 16.03 & 28.52 & \underline{14.65}\% & 19.89& 32.86& 13.37\% & 22.37& 35.70& 9.55 \%& 16.13& 25.52&10.21 \%\\
            LSGCN (IJCAI 2020) & 17.94 &29.85 &16.98 \%&21.53 &33.86 &13.18\% &27.31& 41.46 &11.98\%& 17.73 &26.76 &11.20\%\\
		$^{c}$MTGNN (KDD 2020) & \underline{15.10}& \underline{25.93}& 15.67\% & 19.32& 31.57&13.52\% & 22.07& 35.80& {9.21}\% & 15.71 & \underline{24.62} & 10.03\% \\
		STFGNN (AAAI 2021) & 16.77& 28.34& 16.30\% & 19.83& 31.88&13.02\% & 22.07& 35.80& 9.21\% & 16.64& 26.22& 10.60\%\\
		$^{d}$Z-GCNETs (ICML 2021) & 16.64 & 28.15 & 16.39 \%& 19.50& 31.61& {12.78} \% & {21.77} & {35.17} & 9.25\% & {15.76} & {25.11}& {10.01}\% \\
		STGODE (KDD 2021) & 16.50 &27.84 &16.69\%& 20.84 &32.82 &13.77\% &22.59 &37.54 &10.14\%& 16.81 &25.97& 10.62\% \\
		$^{e}$STG-NCDE (AAAI 2022) & 15.57& 27.09& 15.06\% &19.21 &\underline{31.09} &12.76\%& 20.53 & \underline{33.84} &8.80\%& \underline{15.45}& 24.81& \underline{9.92} \%\\
		$^{f*}$DSTAGNN (ICML 2022) & 15.57& 27.21 &14.68 \% &19.30 &31.46 &12.70\%& 21.42 &34.51 &9.01\%& 15.67& 24.77 & 9.94\%\\
            $^{g}$RGSL (IJCAI 2022) & 15.65 & 27.98 & 14.67\%  & \underline{19.19} & 31.14 & \underline{12.69}\% & 20.73 & 34.48 & \underline{8.71}\% & 15.49 & 24.80 & 9.96\% \\
            \wu{$^{h*}$DGCRN (TKDD 2023)} & 15.98& 27.41& 17.73\% & 20.39& 32.34& 14.64\%& \underline{20.52} & \textbf{33.56} & 9.09\% &16.22 &26.10  & 12.06\%\\
		\hline
		 TrendGCN (ours)& \textbf{14.77} & \textbf{25.66} & \textbf{13.92}\% & \textbf{18.81}& \textbf{30.68}& \textbf{12.25}\% & \textbf{20.43} & 34.32 & \textbf{8.51}\% & \textbf{15.15}& \textbf{24.26}&\textbf{9.51}\% \\ 
        \hline
	\end{tabular}
    }
\end{table*}
\subsection{Experimental Settings}
We first split each dataset into the training set, validation set, and test set by a ratio of 6:2:2 for PEMS03/04/07/08 and a ratio of 7:1:2 for METR-LA/PeMS-BAY. We use the historical one-hour data ($T=12$) to forecast the next-hour data ($H=12$). 
Three metrics are utilized to evaluate model performance, i.e., MAE, RMSE, and MAPE. For the hyper-parameters of TrendGCN, we set the number of hidden units to 64 for GRU cells, GRU layers to 2, GCN layers to 2 by default. The numbers of input features are $F=1$ (flow) for PEMS03/04/07/08 and $F=2$ for METR-LA/PeMS-BAY (speed and time stamps) following \cite{agcrn} and \cite{dcrnn}, respectively. The number of the output feature is $O=1$ for all datasets. We use $\lambda_1=1$, $\lambda_2=1$, and $\lambda_3=1$ in Eq. \ref{eq:node_time_expand} using $\Delta_1=+, \Delta_2=+$ by default. We set $\alpha=0.01$ and $\beta=1.0$ to trade-off the importance of sequence and graph level adversarial training, Adam optimizer with learning rate $\eta=0.003$ and batch size 64, and the \john{spatial and temporal} embedding dimension $d$ are both set to 4, 6, 10, 4, 10, and 10 for PEMS03, PEMS04, PEMS07, PEMS08, METR-LA, and PeMS-Bay datasets, respectively. 
For the experimental results of baselines, we directly cite the best results from their original paper. Otherwise, we report results by running authors-provided source codes under optimal hyper-parameter settings they report in the paper. The experiments are conducted on a computer with a single 24GB NVIDIA GeForce RTX 3090 card. 

\subsection{Performance Comparison and Analysis}
We report our model performance on average 5 times running. The average prediction performances of 12 horizons on PEMS03/04/07/08 are summarized in Table \ref{tab:main_result}, we observe that TrendGCN achieves state-of-the-art on all datasets, except RMSE metrics on the PEMS07 dataset. 
We guess that it is difficult for GANs to discriminate useful signals since the PEMS07 dataset has a large number of traffic nodes (i.e., 883). Besides, we notice that adaptive graph-based methods, e.g., AGCRN, MTGNN, STG-NCDE, RGSL, and TrendGCN(ours) significantly outperform pre-defined graph-based methods, e.g., DCRNN, STGCN, and ASTGCN. The dynamic graph-based methods (DGCRN, DSTAGNN, and TrendGCN(ours)) have an advantage in average predictive performance compared to those using static graphs.
In addition, we compare TrendGCN with other SOTA GAN-based models \cite{timegan,ast,tfgan} on METR-LA and PeMS-Bay. The results in Table. \ref{tab:main_result_2} demonstrates that TrendGCN outperforms best, \wu{which further indicates the effectiveness of modeling dynamics and jointly considering trends and dependencies.}
\begin{table}[t]
	\centering
	\caption{Performance comparison of GAN-based models for traffic speed forecasting on METR-LA and PeMS-BAY datasets with Horizon 12 (60 min).}
	\resizebox{\linewidth}{!}{
	\label{tab:main_result_2}
        \begin{tabular}{c|ccc|ccc}
	    \hline 
        \multirow{2}*{Model} & \multicolumn{3}{c|}{METR-LA} & \multicolumn{3}{c}{PeMS-BAY}\\
    	\cline{2-7}
    	& MAE & RMSE & MAPE & MAE & RMSE & MAPE\\ 
    	\hline
    	\hline
  		TimeGAN (NeurIPS 2019) &4.43 & 8.67 & 13.53\%&2.35&5.16&5.59\%\\ 
		AST  (NeurIPS 2020) & 4.05 & 8.14 &12.80\%&2.27&4.96&5.43\% \\  
		TFGAN (KBS 2022) & 3.83 & 7.98 & 12.72\% & 1.97 & 4.48 & 4.63\%\\
		\hline
		 TrendGCN (ours)& \textbf{3.55} & \textbf{7.39} & \textbf{10.27}\% & \textbf{1.92}& \textbf{4.46}& \textbf{4.51}\%\\ 
        \hline
	\end{tabular}
    }
\end{table}
\begin{table}
    \caption{The impact of different loss objective components in Eq. (\ref{loss_objective}) on prediction performances (MAE/RMSE).}
    \label{tab:loss_ablation}
    \centering
    \resizebox{0.8\linewidth}{!}{
    \begin{tabular}{c|c|c|c}
    \hline
    & TrendGCN & w/o $\mathcal{L}_{\text{adv}}$ & w/o $\mathcal{L}_{p}(\bm{\Theta})$ \\
    \hline
    \hline
     PEMS04 & 18.81/30.68 & 19.04/31.62 & 43.32/64.89 \\
     \hline
    \end{tabular}
    }
\end{table}
\begin{table}[]
    \caption{Complexity and execution efficiency analysis of models on PEMS04 dataset.} 
    \label{tab:complexity}
    \centering
    \resizebox{\linewidth}{!}{
    \begin{tabular}{l|c|c|c}
    \hline
    PEMS04 & TrendGCN & RGSL (IJCAI 2022) & DSTAGNN (ICML 2022) \\ 
    \hline
    \hline
     \# Parameters & 0.45M & 0.87M & 3.58M \\
     \# GPU Memory & 5.38GB & 7.72GB & 8.77GB \\
     Training Cost (epoch) & 49.32s & 61.01s & 116.20s \\
     Inference Cost (epoch) & 1.83s & 3.11s & 10.02s \\
     Complexity (per Layer) & 
     $\bm{\mathcal{O}}(N^2d+Td^2)$ &
     $\bm{\mathcal{O}}(N^2d+Td^2)$ & 
     $\bm{\mathcal{O}}(N^2d+kTd^2)$\\
     \hline
     MAE/RMSE & 18.81/30.68 & 19.19/31.14 &19.30/31.46 \\
     \hline
    \end{tabular}
    }
\end{table}
\begin{figure*}[t]
	\centering
	\includegraphics[width=0.9\linewidth]{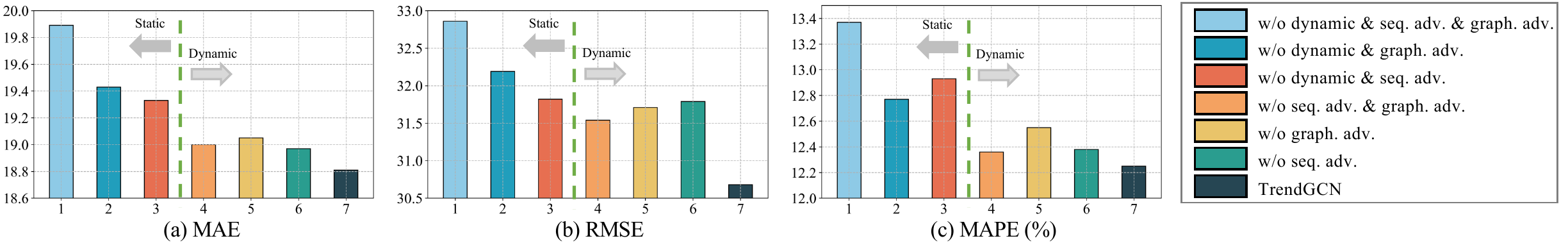}
	\caption{Ablation study of our TrendGCN with(w) or without(w/o) proposed components on PEMS04 dataset.}
	\label{fig:ablation}
\end{figure*}
\begin{figure*}[t]
	\centering
	\includegraphics[width=0.9\linewidth]{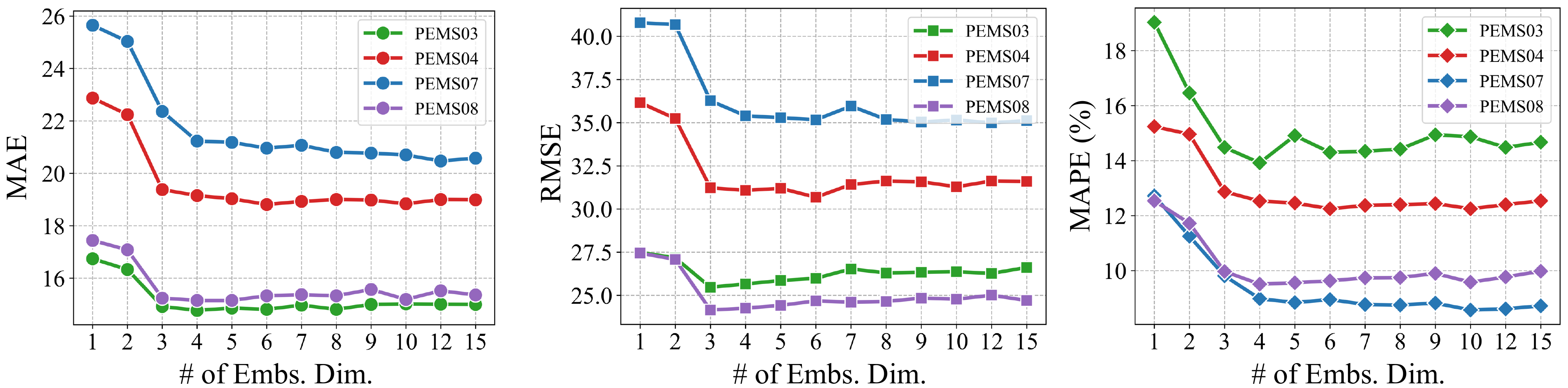}
	\caption{Influence of representation dimensions of \john{the spatial and temporal} embeddings on PEMS03/04/07/08 datasets.}
	\label{fig:ablation_dim}
\end{figure*}
\subsection{Ablation Study}
We conduct an ablation study with its variants to verify the effectiveness of each component in TrendGCN. As shown in Fig. \ref{fig:ablation}, variants with dynamic graphs outperform the ones with a static graph. Besides, adversarial training significantly improves the prediction performance of all variants. \wu{Adversarial training at the graph level is better than at the sequence level, which implies that the dependencies between all nodes may play a stronger role in eliminating discrepancies.}
\wbq{In addition, we compare adversarial loss $\mathcal{L}_{\text{adv}}$ with $\mathcal{L}_{p}(\bm{\Theta})$ in Table \ref{tab:loss_ablation}. It demonstrates that removing either $\mathcal{L}_{\text{adv}}$ or $\mathcal{L}_{p}(\bm{\Theta})$ will result in a drop in prediction performance, and $\mathcal{L}_{p}(\bm{\Theta})$ plays a vital role in supervised learning}. 

\subsection{Hyperparameter Study}
Since the embedding dimension (i.e., $d_e$) of spatial embeddings and temporal embeddings has a great impact on model performance and computational cost, we present prediction performance at different settings, as shown in Fig. \ref{fig:ablation_dim}. \wbq{We observe that the basic principle is that $d$ should not be set too small (insufficient representation) and too large (over-fitting and time-consuming problem)}. The optimal embedding dimension should be set as 4, 6, 10, and 4 for PEMS03, PEMS04, PEMS07, and PEMS08 datasets, respectively. 
In addition, since adversarial learning is sensitive to weights, we discuss the influence of loss trade-off weights $\alpha$ and $\beta$ of $ \mathcal{L}_{\text{adv}}$ in Fig. \ref{fig:loss}(a). We find that on most datasets, the MSE is relatively stable when the trade off ratios are in the range $[0, 0.01, 0.05, 0.1]$. 
\begin{figure}[t]
	\centering
        \includegraphics[width=\linewidth]{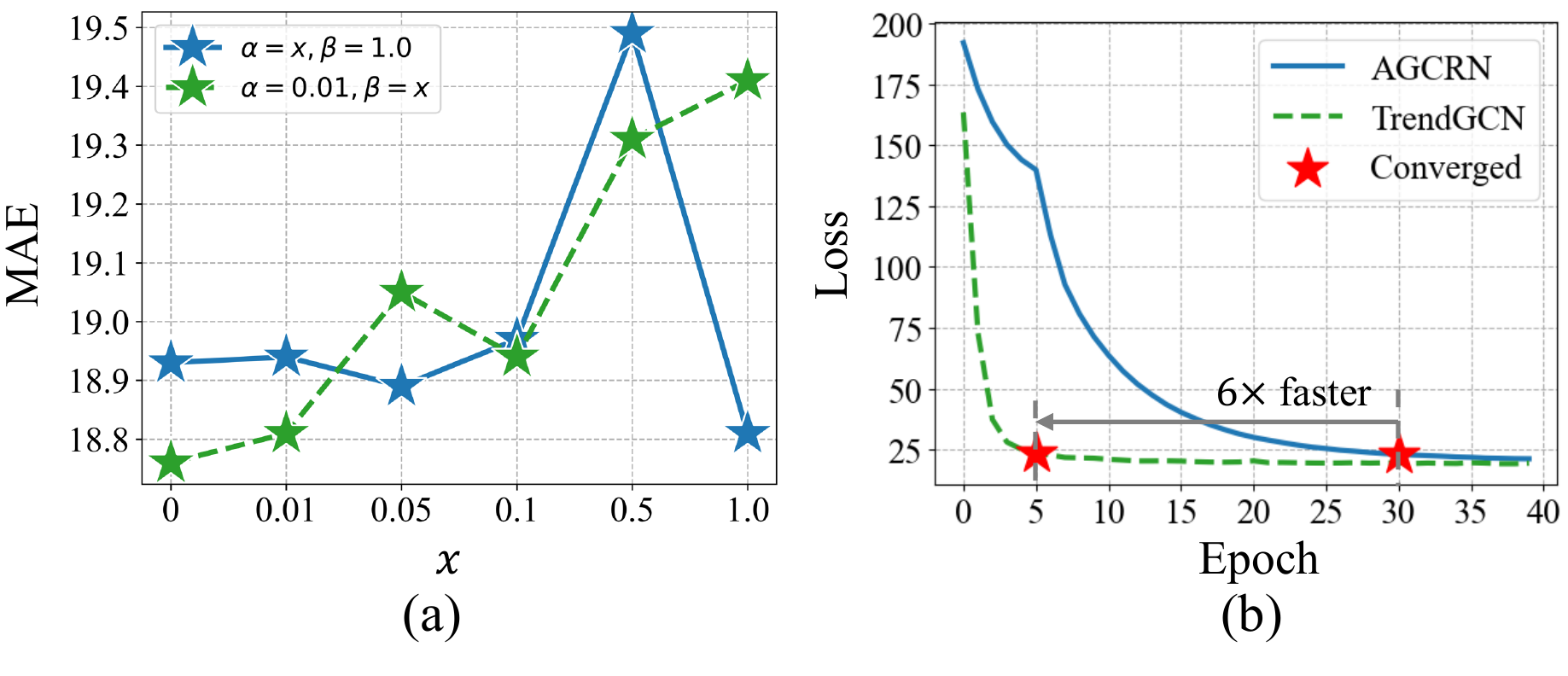}
	\caption{(a) The impact of loss trade-off weights $\alpha$ and $\beta$ of $ \mathcal{L}_{\text{adv}}$. (b) The Convergence speed comparison with AGCRN. Both on PEMS04 dataset.}
	\label{fig:loss}
\end{figure}
\begin{figure}[t]
	\centering
	\includegraphics[width=\linewidth]{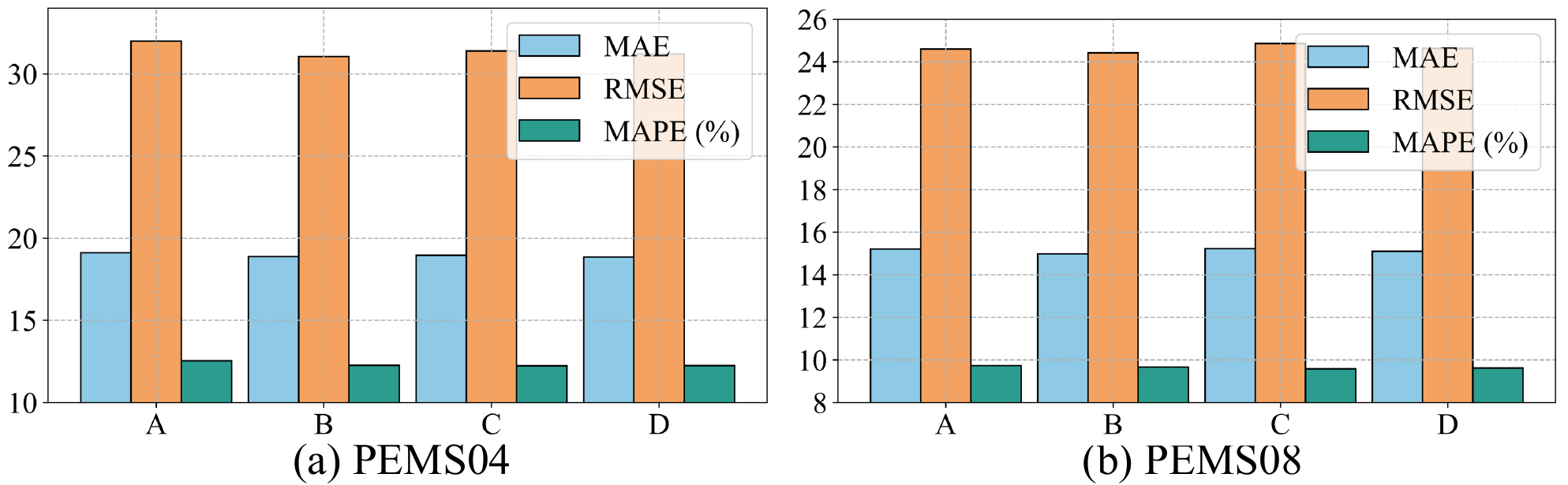}
	\caption{Performance comparison of four widely-used combinations for graph construction under the unified scheme on PEMS04/08 datasets.}
    \label{tab:coefficient}
\end{figure}
\begin{table}[]
    \caption{Prediction error (MAE/RMSE) of different methods on original data (1st row), Gaussian-noise polluted data (2nd row), and the relative increment ratio of the error (3rd row, smaller is better).
    }
    \label{tab:robustness}
    \centering
    \resizebox{\linewidth}{!}{
    \begin{tabular}{c|c|c|c}
    \hline
    & TrendGCN & RGSL (IJCAI 2022) & DSTAGNN (ICML 2022) \\
    \hline
    \hline
     PEMS04 & 18.81/30.68 & 19.19/31.14 &19.30/31.46\\
     $+\mathcal{N}(0, 1)$ & 24.91/37.36 & 27.98/40.81 & 27.22/40.28\\ 
     $+\Delta$errors & +32.43\%/+21.77\%& +45.81\%/+31.05\%& +41.04\%/+28.04\% \\
     \hline
    \end{tabular}
    }
\end{table}
\wbq{\subsection{Complexity Analysis and Cost}
To compare the computation cost of TrendGCN and SOTA, we show their complexity and execution efficiency in Table \ref{tab:complexity} and Fig. \ref{fig:loss}. As can be seen, our approach has better efficiency in both training (12\%-50\% less time) and inference (50\%-80\% less time), and smaller memory footprint (20\%-30\% less) compared with SOTA. The results indicate that TrendGCN can achieve a good trade-off between computational cost and forecasting accuracy. Besides, our TrendGCN accomplishes an average of 6 times faster convergence speed compared with AGCRN, as shown in Fig. \ref{fig:loss}(b).}

\begin{figure*}[t]
	\centering
	\includegraphics[width=0.9\linewidth]{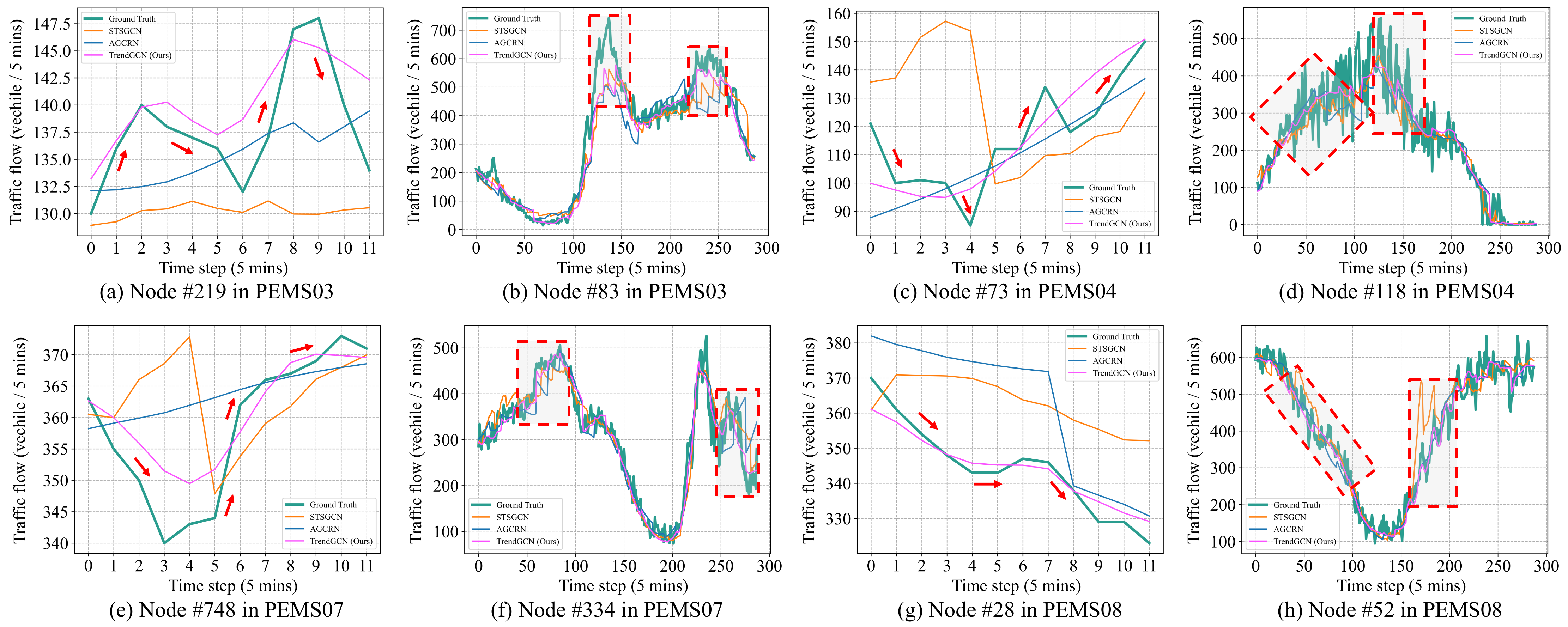}
	\caption{Comparison of short (12 steps)-(a)(c)(e)(g) and long (288 steps)-(b)(d)(f)(h) term prediction curves between STSGCN, AGCRN, and our TrendGCN on a snapshot of the test data of four datasets. Note that, the predicted time series for the whole day period (288 steps) is simply obtained by concatenating all the short-term predictions (12 steps) along the time axis (and remove overlaps), which is a common practice widely used in \cite{dstagnn,li2021spatial,agcrn}, so that a better visualization of the prediction quality during different time of the day can be presented. }
	\label{fig:loss_horizon}
\end{figure*}
\begin{figure*}[t]
\centering
\includegraphics[width=\linewidth]{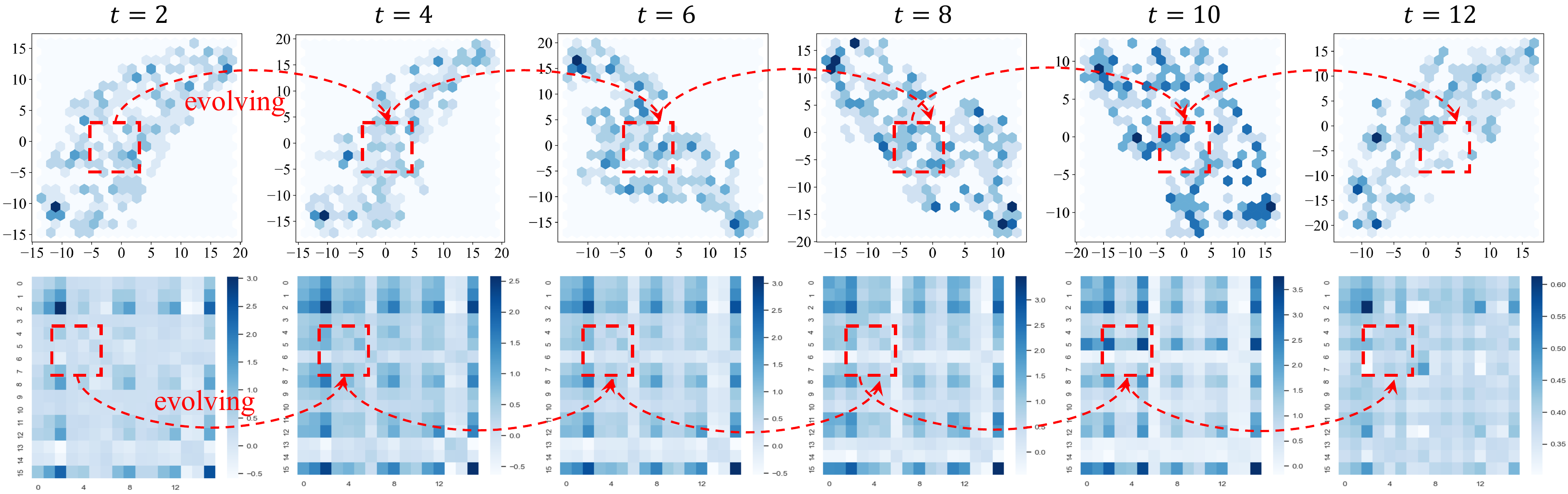}
    \caption{Visualization of 2D projection of UMAP on spatial embeddings (Upper) and the heatmap of learned graphs (Lower) at $t\in\{2, 4, 6, 8, 10, 12\}$ time steps.}
\label{fig:dynamic_graph}
\end{figure*}

\subsection{Robustness Exploration}
\wbq{To test the robustness of TrendGCN, we conduct experiments by injecting Gaussian noises into the raw traffic data of PEMS04 dataset. The results in Table \ref{tab:robustness} show the increasing errors of TrendGCN are much less than SOTA for the polluted data, verifying the robustness of TrendGCN. One of the possible reasons for such results is that TrendGCN can capture the global trend and local dynamics of traffic data, which helps to reduce the risk of local over-fitting.}
\subsection{Visualization}
We compare the short (12 steps) and long (288 steps) term prediction curves between STSGCN, AGCRN, and our TrendGCN on a snapshot of the test data of four datasets, as shown in Fig. \ref{fig:loss_horizon}. 
We observe that our proposed TrendGCN can significantly bridge the trend discrepancy between prediction and ground truth for both short-term and long-term prediction, which confirms our intuition. In particular, for the fast-varying periods (dashed boxes), the predictions of TrendGCN are much closer to ground truth, which shows the stronger adaptive ability of TrendGCN for changes.
\wbq{Furthermroe, we visualize  the learned dynamic adaptive graphs at the different time steps, aiming to discuss the interpretation of TrendGCN. For better visualization, we randomly select 16 nodes on PEMS04 dataset, as shown in Fig. \ref{fig:dynamic_graph}. We have the following observations: 1) Although many methods using pre-defined graphs (static) have achieved comparable performance, they generally face the problem of data sparsity which harms the propagation of model's gradient significantly; 2) Dynamic adaptive graphs can flexibly capture the complex spatial-temporal dependencies between all nodes at different time steps.}

\subsection{Graph Construction Discussion}
We propose a unified scheme (see Eq. \ref{eq:node_time_1}) to effectively couple the spatial (node-wise) and temporal (time-wise) embeddings through a gate module and use the integrated embeddings to construct graphs changing over time.
The choices of $\Delta_1, \Delta_2$ can be the same or different. 
Here, four widely-used combinations with $\bm{\lambda}^1 = \bm{\lambda}^2 = \bm{\lambda}^3 = \bm{\lambda}^4 = 1$ are discussed as follows:
\begin{equation}
\label{eq:node_time_2}
\begin{aligned}
^A\bm{\mathcal{A}}_{ij}^{(t)} &= \bm{\lambda}^1\left \langle \text{Dpt}\left(\text{LN}(\bm{e}_{{\text{node}}}^{(i)} \mathbin{\|} \bm{e}_{{\text{time}}}^{(t)})\right), \text{Dpt}\left(\text{LN}(\bm{e}_{{\text{node}}}^{(j)} \mathbin{\|} \bm{e}_{{\text{time}}}^{(t)})\right)\right \rangle \\
^B\bm{\mathcal{A}}_{ij}^{(t)} &= \bm{\lambda}^2\left \langle \text{Dpt}\left(\text{LN}(\bm{e}_{{\text{node}}}^{(i)} \odot \bm{e}_{{\text{time}}}^{(t)})\right), \text{Dpt}\left(\text{LN}(\bm{e}_{{\text{node}}}^{(j)} \odot \bm{e}_{{\text{time}}}^{(t)})\right)\right \rangle \\
^C\bm{\mathcal{A}}_{ij}^{(t)} &= \bm{\lambda}^3\left \langle \text{Dpt}\left(\text{LN}(\bm{e}_{{\text{node}}}^{(i)} + \bm{e}_{{\text{time}}}^{(t)})\right), \text{Dpt}\left(\text{LN}(\bm{e}_{{\text{node}}}^{(j)} \odot \bm{e}_{{\text{time}}}^{(t)})\right)\right \rangle \\
^D\bm{\mathcal{A}}_{ij}^{(t)} &= \bm{\lambda}^4\left \langle \text{Dpt}\left(\text{LN}(\bm{e}_{{\text{node}}}^{(i)} \odot \bm{e}_{{\text{time}}}^{(t)})\right), \text{Dpt}\left(\text{LN}(\bm{e}_{{\text{node}}}^{(j)} + \bm{e}_{{\text{time}}}^{(t)})\right)\right \rangle 
\end{aligned}
\end{equation}
As can be seen in Fig. \ref{tab:coefficient}, we derive the following findings: (1) The default setting of $\Delta_1=+, \Delta_2=+$ 
in TrendGCN 
achieves optimal performance (see Table \ref{tab:main_result}), which signifies the equal importance of homogeneous and heterogeneous interactions in the spatial-temporal domains. (2) TrendGCN is not sensitive to the choices of $\Delta_1, \Delta_2$ (the color bars are almost the same height) which further verifies our method enhances the robustness of traffic forecasting.

\section{Conclusions and Future Work}
In this paper, we proposed TrendGCN, a novel model for traffic forecasting that extends the flexibility of GCNs and the distribution-preserving capacity of generative and adversarial loss. Our approach addresses the challenges of capturing dynamics and maintaining robustness by introducing dynamic adaptive graph generation and adversarial dynamic trend alignment. Extensive experiments on six benchmarks and theoretical analyses demonstrate the superiority of TrendGCN. For further work, we will study the following two aspects: 1) investigating stronger methods to capture dynamic spatial-temporal dependencies, e.g., the mixture of experts (MoE); 2) exploring more effective approaches to 
enhance the robustness of traffic forecasting, e.g., taking higher-order derivatives of time series.


\section{Acknowledgments}
This work was partially supported by the National Natural Science Foundation of China (No. 62276099).

\bibliographystyle{ACM-Reference-Format} 
\balance
\bibliography{ref} 


\begin{thebibliography}{44}


\ifx \showCODEN    \undefined \def \showCODEN     #1{\unskip}     \fi
\ifx \showDOI      \undefined \def \showDOI       #1{#1}\fi
\ifx \showISBNx    \undefined \def \showISBNx     #1{\unskip}     \fi
\ifx \showISBNxiii \undefined \def \showISBNxiii  #1{\unskip}     \fi
\ifx \showISSN     \undefined \def \showISSN      #1{\unskip}     \fi
\ifx \showLCCN     \undefined \def \showLCCN      #1{\unskip}     \fi
\ifx \shownote     \undefined \def \shownote      #1{#1}          \fi
\ifx \showarticletitle \undefined \def \showarticletitle #1{#1}   \fi
\ifx \showURL      \undefined \def \showURL       {\relax}        \fi
\providecommand\bibfield[2]{#2}
\providecommand\bibinfo[2]{#2}
\providecommand\natexlab[1]{#1}
\providecommand\showeprint[2][]{arXiv:#2}

\bibitem[Bai et~al\mbox{.}(2019)]%
        {stg2seq}
\bibfield{author}{\bibinfo{person}{Lei Bai}, \bibinfo{person}{Lina Yao},
  \bibinfo{person}{Salil~S Kanhere}, \bibinfo{person}{Xianzhi Wang}, {and}
  \bibinfo{person}{Quan~Z Sheng}.} \bibinfo{year}{2019}\natexlab{}.
\newblock \showarticletitle{STG2seq: spatial-temporal graph to sequence model
  for multi-step passenger demand forecasting}. In
  \bibinfo{booktitle}{\emph{Proceedings of the International Joint Conference
  on Artificial Intelligence}}. \bibinfo{publisher}{AAAI Press},
  \bibinfo{address}{Palo Alto, CA USA}, \bibinfo{pages}{1981--1987}.
\newblock


\bibitem[Bai et~al\mbox{.}(2020)]%
        {agcrn}
\bibfield{author}{\bibinfo{person}{Lei Bai}, \bibinfo{person}{Lina Yao},
  \bibinfo{person}{Can Li}, \bibinfo{person}{Xianzhi Wang}, {and}
  \bibinfo{person}{Can Wang}.} \bibinfo{year}{2020}\natexlab{}.
\newblock \showarticletitle{Adaptive graph convolutional recurrent network for
  traffic forecasting}. In \bibinfo{booktitle}{\emph{Advances in Neural
  Information Processing Systems}}. \bibinfo{publisher}{MIT Press},
  \bibinfo{pages}{17804--17815}.
\newblock


\bibitem[Bai et~al\mbox{.}(2018)]%
        {tcn}
\bibfield{author}{\bibinfo{person}{Shaojie Bai}, \bibinfo{person}{J~Zico
  Kolter}, {and} \bibinfo{person}{Vladlen Koltun}.}
  \bibinfo{year}{2018}\natexlab{}.
\newblock \showarticletitle{An empirical evaluation of generic convolutional
  and recurrent networks for sequence modeling}. In
  \bibinfo{booktitle}{\emph{International Conference on Learning
  Representations Workshop}}.
\newblock


\bibitem[Cao et~al\mbox{.}(2020)]%
        {stemgnn}
\bibfield{author}{\bibinfo{person}{Defu Cao}, \bibinfo{person}{Yujing Wang},
  \bibinfo{person}{Juanyong Duan}, \bibinfo{person}{Ce Zhang},
  \bibinfo{person}{Xia Zhu}, \bibinfo{person}{Congrui Huang},
  \bibinfo{person}{Yunhai Tong}, \bibinfo{person}{Bixiong Xu},
  \bibinfo{person}{Jing Bai}, \bibinfo{person}{Jie Tong}, {and}
  \bibinfo{person}{Qi Zhang}.} \bibinfo{year}{2020}\natexlab{}.
\newblock \showarticletitle{Spectral temporal graph neural network for
  multivariate time-series forecasting}.
\newblock \bibinfo{journal}{\emph{Advances in Neural Information Processing
  Systems}} (\bibinfo{year}{2020}), \bibinfo{pages}{17766--17778}.
\newblock


\bibitem[Chen et~al\mbox{.}(2023)]%
        {magnn}
\bibfield{author}{\bibinfo{person}{Ling Chen}, \bibinfo{person}{Donghui Chen},
  \bibinfo{person}{Zongjiang Shang}, \bibinfo{person}{Binqing Wu},
  \bibinfo{person}{Cen Zheng}, \bibinfo{person}{Bo Wen}, {and}
  \bibinfo{person}{Wei Zhang}.} \bibinfo{year}{2023}\natexlab{}.
\newblock \showarticletitle{Multi-scale adaptive graph neural network for
  multivariate time series forecasting}.
\newblock \bibinfo{journal}{\emph{IEEE Transactions on Knowledge and Data
  Engineering}} (\bibinfo{year}{2023}).
\newblock


\bibitem[Chen et~al\mbox{.}(2021)]%
        {z_gcnnets}
\bibfield{author}{\bibinfo{person}{Yuzhou Chen}, \bibinfo{person}{Ignacio
  Segovia}, {and} \bibinfo{person}{Yulia~R Gel}.}
  \bibinfo{year}{2021}\natexlab{}.
\newblock \showarticletitle{{Z-GCNETs}: Time zigzags at graph convolutional
  networks for time series forecasting}. In
  \bibinfo{booktitle}{\emph{International Conference on Machine Learning}}.
  \bibinfo{pages}{1684--1694}.
\newblock


\bibitem[Choi et~al\mbox{.}(2022)]%
        {stgbcde}
\bibfield{author}{\bibinfo{person}{Jeongwhan Choi}, \bibinfo{person}{Hwangyong
  Choi}, \bibinfo{person}{Jeehyun Hwang}, {and} \bibinfo{person}{Noseong
  Park}.} \bibinfo{year}{2022}\natexlab{}.
\newblock \showarticletitle{Graph neural controlled differential equations for
  traffic forecasting}. In \bibinfo{booktitle}{\emph{Proceedings of the AAAI
  Conference on Artificial Intelligence}}, Vol.~\bibinfo{volume}{36}.
  \bibinfo{publisher}{AAAI Press}, \bibinfo{address}{Palo Alto, CA USA},
  \bibinfo{pages}{6367--6374}.
\newblock


\bibitem[Chung et~al\mbox{.}(2014)]%
        {chung2014empirical}
\bibfield{author}{\bibinfo{person}{Junyoung Chung}, \bibinfo{person}{Caglar
  Gulcehre}, \bibinfo{person}{KyungHyun Cho}, {and} \bibinfo{person}{Yoshua
  Bengio}.} \bibinfo{year}{2014}\natexlab{}.
\newblock \showarticletitle{Empirical evaluation of gated recurrent neural
  networks on sequence modeling}. In \bibinfo{booktitle}{\emph{Advances in
  Neural Information Processing Systems Workshop}}. \bibinfo{publisher}{MIT
  Press}.
\newblock


\bibitem[Dai et~al\mbox{.}(2019)]%
        {dai2019transformer}
\bibfield{author}{\bibinfo{person}{Zihang Dai}, \bibinfo{person}{Zhilin Yang},
  \bibinfo{person}{Yiming Yang}, \bibinfo{person}{Jaime Carbonell},
  \bibinfo{person}{Quoc~V Le}, {and} \bibinfo{person}{Ruslan Salakhutdinov}.}
  \bibinfo{year}{2019}\natexlab{}.
\newblock \showarticletitle{Transformer-xl: Attentive language models beyond a
  fixed-length context}. In \bibinfo{booktitle}{\emph{Proceedings of the Annual
  Meeting of the Association for Computational Linguistics}}.
  \bibinfo{pages}{2978--2988}.
\newblock


\bibitem[Diao et~al\mbox{.}(2019)]%
        {dgcnn}
\bibfield{author}{\bibinfo{person}{Zulong Diao}, \bibinfo{person}{Xin Wang},
  \bibinfo{person}{Dafang Zhang}, \bibinfo{person}{Yingru Liu},
  \bibinfo{person}{Kun Xie}, {and} \bibinfo{person}{Shaoyao He}.}
  \bibinfo{year}{2019}\natexlab{}.
\newblock \showarticletitle{Dynamic spatial-temporal graph convolutional neural
  networks for traffic forecasting}. In \bibinfo{booktitle}{\emph{Proceedings
  of the AAAI conference on artificial intelligence}},
  Vol.~\bibinfo{volume}{33}. \bibinfo{publisher}{AAAI Press},
  \bibinfo{address}{Palo Alto, CA USA}, \bibinfo{pages}{890--897}.
\newblock


\bibitem[Fang et~al\mbox{.}(2021)]%
        {stgode}
\bibfield{author}{\bibinfo{person}{Zheng Fang}, \bibinfo{person}{Qingqing
  Long}, \bibinfo{person}{Guojie Song}, {and} \bibinfo{person}{Kunqing Xie}.}
  \bibinfo{year}{2021}\natexlab{}.
\newblock \showarticletitle{Spatial-temporal graph ode networks for traffic
  flow forecasting}. In \bibinfo{booktitle}{\emph{Proceedings of the ACM SIGKDD
  International Conference on Knowledge Discovery \& Data Mining}}.
  \bibinfo{publisher}{Association for Computing Machinery},
  \bibinfo{address}{New York, NY, USA}, \bibinfo{pages}{364--373}.
\newblock


\bibitem[Geng et~al\mbox{.}(2019)]%
        {stmgcn}
\bibfield{author}{\bibinfo{person}{Xu Geng}, \bibinfo{person}{Yaguang Li},
  \bibinfo{person}{Leye Wang}, \bibinfo{person}{Lingyu Zhang},
  \bibinfo{person}{Qiang Yang}, \bibinfo{person}{Jieping Ye}, {and}
  \bibinfo{person}{Yan Liu}.} \bibinfo{year}{2019}\natexlab{}.
\newblock \showarticletitle{Spatiotemporal multi-graph convolution network for
  ride-hailing demand forecasting}. In \bibinfo{booktitle}{\emph{Proceedings of
  the AAAI Conference on Artificial Intelligence}}. \bibinfo{publisher}{AAAI
  Press}, \bibinfo{address}{Palo Alto, CA USA}, \bibinfo{pages}{3656--3663}.
\newblock


\bibitem[Guo et~al\mbox{.}(2019)]%
        {astgcn}
\bibfield{author}{\bibinfo{person}{Shengnan Guo}, \bibinfo{person}{Youfang
  Lin}, \bibinfo{person}{Ning Feng}, \bibinfo{person}{Chao Song}, {and}
  \bibinfo{person}{Huaiyu Wan}.} \bibinfo{year}{2019}\natexlab{}.
\newblock \showarticletitle{Attention based spatial-temporal graph
  convolutional networks for traffic flow forecasting}. In
  \bibinfo{booktitle}{\emph{Proceedings of the AAAI Conference on Artificial
  Intelligence}}. \bibinfo{publisher}{AAAI Press}, \bibinfo{address}{Palo Alto,
  CA USA}, \bibinfo{pages}{922--929}.
\newblock


\bibitem[Guo et~al\mbox{.}(2021)]%
        {astgnn}
\bibfield{author}{\bibinfo{person}{Shengnan Guo}, \bibinfo{person}{Youfang
  Lin}, \bibinfo{person}{Huaiyu Wan}, \bibinfo{person}{Xiucheng Li}, {and}
  \bibinfo{person}{Gao Cong}.} \bibinfo{year}{2021}\natexlab{}.
\newblock \showarticletitle{Learning dynamics and heterogeneity of
  spatial-temporal graph data for traffic forecasting}.
\newblock \bibinfo{journal}{\emph{IEEE Transactions on Knowledge and Data
  Engineering}} (\bibinfo{year}{2021}).
\newblock


\bibitem[Haidar et~al\mbox{.}(2019)]%
        {gan_nlp}
\bibfield{author}{\bibinfo{person}{Md Haidar}, \bibinfo{person}{Mehdi
  Rezagholizadeh}, {et~al\mbox{.}}} \bibinfo{year}{2019}\natexlab{}.
\newblock \showarticletitle{Textkd-{GAN}: Text generation using knowledge
  distillation and generative adversarial networks}. In
  \bibinfo{booktitle}{\emph{Canadian conference on artificial intelligence}}.
  \bibinfo{pages}{107--118}.
\newblock


\bibitem[Huang et~al\mbox{.}(2020)]%
        {lsgcn}
\bibfield{author}{\bibinfo{person}{Rongzhou Huang}, \bibinfo{person}{Chuyin
  Huang}, \bibinfo{person}{Yubao Liu}, \bibinfo{person}{Genan Dai}, {and}
  \bibinfo{person}{Weiyang Kong}.} \bibinfo{year}{2020}\natexlab{}.
\newblock \showarticletitle{LSGCN: Long Short-Term Traffic Prediction with
  Graph Convolutional Networks}. In \bibinfo{booktitle}{\emph{Proceedings of
  the International Joint Conference on Artificial Intelligence}}.
  \bibinfo{publisher}{AAAI Press}, \bibinfo{address}{Palo Alto, CA USA},
  \bibinfo{pages}{2355--2361}.
\newblock


\bibitem[Jiang et~al\mbox{.}(2020)]%
        {jiang2020cascaded}
\bibfield{author}{\bibinfo{person}{Juyong Jiang}, \bibinfo{person}{Jie Zhang},
  {and} \bibinfo{person}{Kai Zhang}.} \bibinfo{year}{2020}\natexlab{}.
\newblock \showarticletitle{Cascaded semantic and positional self-attention
  network for document classification}. In \bibinfo{booktitle}{\emph{Conference
  on Empirical Methods in Natural Language Processing}}.
  \bibinfo{publisher}{Association for Computational Linguistics},
  \bibinfo{pages}{669--677}.
\newblock


\bibitem[Jiang et~al\mbox{.}(2021)]%
        {dl_traffic_survey}
\bibfield{author}{\bibinfo{person}{Renhe Jiang}, \bibinfo{person}{Du Yin},
  \bibinfo{person}{Zhaonan Wang}, \bibinfo{person}{Yizhuo Wang},
  \bibinfo{person}{Jiewen Deng}, \bibinfo{person}{Hangchen Liu},
  \bibinfo{person}{Zekun Cai}, \bibinfo{person}{Jinliang Deng},
  \bibinfo{person}{Xuan Song}, {and} \bibinfo{person}{Ryosuke Shibasaki}.}
  \bibinfo{year}{2021}\natexlab{}.
\newblock \showarticletitle{DL-Traff: Survey and benchmark of deep learning
  models for urban traffic prediction}. In
  \bibinfo{booktitle}{\emph{Proceedings of the ACM International Conference on
  Information \& Knowledge Management}}. \bibinfo{publisher}{Association for
  Computing Machinery}, \bibinfo{address}{New York, NY, USA},
  \bibinfo{pages}{4515--4525}.
\newblock


\bibitem[Jiang and Luo(2021)]%
        {traffi_gcn_survey}
\bibfield{author}{\bibinfo{person}{Weiwei Jiang} {and} \bibinfo{person}{Jiayun
  Luo}.} \bibinfo{year}{2021}\natexlab{}.
\newblock \showarticletitle{Graph neural network for traffic forecasting: {A}
  Survey}.
\newblock \bibinfo{journal}{\emph{arXiv}} (\bibinfo{year}{2021}).
\newblock


\bibitem[Khaled et~al\mbox{.}(2022)]%
        {tfgan}
\bibfield{author}{\bibinfo{person}{Alkilane Khaled}, \bibinfo{person}{Alfateh
  M~Tag Elsir}, {and} \bibinfo{person}{Yanming Shen}.}
  \bibinfo{year}{2022}\natexlab{}.
\newblock \showarticletitle{TFGAN: Traffic forecasting using generative
  adversarial network with multi-graph convolutional network}.
\newblock \bibinfo{journal}{\emph{Knowledge-Based Systems}}
  \bibinfo{volume}{249} (\bibinfo{year}{2022}), \bibinfo{pages}{108990}.
\newblock


\bibitem[Lan et~al\mbox{.}(2022)]%
        {dstagnn}
\bibfield{author}{\bibinfo{person}{Shiyong Lan}, \bibinfo{person}{Yitong Ma},
  \bibinfo{person}{Weikang Huang}, \bibinfo{person}{Wenwu Wang},
  \bibinfo{person}{Hongyu Yang}, {and} \bibinfo{person}{Pyang Li}.}
  \bibinfo{year}{2022}\natexlab{}.
\newblock \showarticletitle{DSTAGNN: Dynamic Spatial-Temporal Aware Graph
  Neural Network for Traffic Flow Forecasting}. In
  \bibinfo{booktitle}{\emph{International Conference on Machine Learning}}.
  PMLR, \bibinfo{pages}{11906--11917}.
\newblock


\bibitem[Li et~al\mbox{.}(2023)]%
        {dgcrn}
\bibfield{author}{\bibinfo{person}{Fuxian Li}, \bibinfo{person}{Jie Feng},
  \bibinfo{person}{Huan Yan}, \bibinfo{person}{Guangyin Jin},
  \bibinfo{person}{Fan Yang}, \bibinfo{person}{Funing Sun},
  \bibinfo{person}{Depeng Jin}, {and} \bibinfo{person}{Yong Li}.}
  \bibinfo{year}{2023}\natexlab{}.
\newblock \showarticletitle{Dynamic Graph Convolutional Recurrent Network for
  Traffic Prediction: Benchmark and Solution}.
\newblock \bibinfo{journal}{\emph{ACM Transactions on Knowledge Discovery from
  Data (TKDD)}} \bibinfo{volume}{17}, \bibinfo{number}{1}
  (\bibinfo{year}{2023}).
\newblock


\bibitem[Li and Zhu(2021a)]%
        {stfgcn}
\bibfield{author}{\bibinfo{person}{Mengzhang Li} {and}
  \bibinfo{person}{Zhanxing Zhu}.} \bibinfo{year}{2021}\natexlab{a}.
\newblock \showarticletitle{Spatial-temporal fusion graph neural networks for
  traffic flow forecasting}. In \bibinfo{booktitle}{\emph{Proceedings of the
  AAAI Conference on Artificial Intelligence}}. \bibinfo{publisher}{AAAI
  Press}, \bibinfo{address}{Palo Alto, CA USA}, \bibinfo{pages}{4189--4196}.
\newblock


\bibitem[Li and Zhu(2021b)]%
        {li2021spatial}
\bibfield{author}{\bibinfo{person}{Mengzhang Li} {and}
  \bibinfo{person}{Zhanxing Zhu}.} \bibinfo{year}{2021}\natexlab{b}.
\newblock \showarticletitle{Spatial-temporal fusion graph neural networks for
  traffic flow forecasting}. In \bibinfo{booktitle}{\emph{Proceedings of the
  AAAI Conference on Artificial Intelligence}}, Vol.~\bibinfo{volume}{35}.
  \bibinfo{publisher}{AAAI Press}, \bibinfo{address}{Palo Alto, CA USA},
  \bibinfo{pages}{4189--4196}.
\newblock


\bibitem[Li et~al\mbox{.}(2018)]%
        {dcrnn}
\bibfield{author}{\bibinfo{person}{Yaguang Li}, \bibinfo{person}{Rose Yu},
  \bibinfo{person}{Cyrus Shahabi}, {and} \bibinfo{person}{Yan Liu}.}
  \bibinfo{year}{2018}\natexlab{}.
\newblock \showarticletitle{Diffusion convolutional recurrent neural network:
  Data-driven traffic forecasting}. In \bibinfo{booktitle}{\emph{International
  Conference on Learning Representations}}.
\newblock


\bibitem[Ma et~al\mbox{.}(2017)]%
        {cnn-rnn}
\bibfield{author}{\bibinfo{person}{Xiaolei Ma}, \bibinfo{person}{Zhuang Dai},
  \bibinfo{person}{Zhengbing He}, \bibinfo{person}{Jihui Ma},
  \bibinfo{person}{Yong Wang}, {and} \bibinfo{person}{Yunpeng Wang}.}
  \bibinfo{year}{2017}\natexlab{}.
\newblock \showarticletitle{Learning traffic as images: A deep convolutional
  neural network for large-scale transportation network speed prediction}.
\newblock \bibinfo{journal}{\emph{Sensors}} \bibinfo{volume}{17},
  \bibinfo{number}{4} (\bibinfo{year}{2017}), \bibinfo{pages}{818}.
\newblock


\bibitem[SHI et~al\mbox{.}(2015)]%
        {fc_lstm}
\bibfield{author}{\bibinfo{person}{Xingjian SHI}, \bibinfo{person}{Zhourong
  Chen}, \bibinfo{person}{Hao Wang}, \bibinfo{person}{Dit-Yan Yeung},
  \bibinfo{person}{Wai-kin Wong}, {and} \bibinfo{person}{Wang-chun WOO}.}
  \bibinfo{year}{2015}\natexlab{}.
\newblock \showarticletitle{Convolutional LSTM network: A machine learning
  approach for precipitation nowcasting}. In \bibinfo{booktitle}{\emph{Advances
  in Neural Information Processing Systems}}. \bibinfo{publisher}{MIT Press},
  \bibinfo{pages}{802–810}.
\newblock


\bibitem[Song et~al\mbox{.}(2020)]%
        {stsgcn}
\bibfield{author}{\bibinfo{person}{Chao Song}, \bibinfo{person}{Youfang Lin},
  \bibinfo{person}{Shengnan Guo}, {and} \bibinfo{person}{Huaiyu Wan}.}
  \bibinfo{year}{2020}\natexlab{}.
\newblock \showarticletitle{Spatial-temporal synchronous graph convolutional
  networks: A new framework for spatial-temporal network data forecasting}. In
  \bibinfo{booktitle}{\emph{Proceedings of the AAAI Conference on Artificial
  Intelligence}}. \bibinfo{publisher}{AAAI Press}, \bibinfo{address}{Palo Alto,
  CA USA}, \bibinfo{pages}{914--921}.
\newblock


\bibitem[Thomas N.~Kipf(2017)]%
        {gnn}
\bibfield{author}{\bibinfo{person}{Max~Welling Thomas N.~Kipf}.}
  \bibinfo{year}{2017}\natexlab{}.
\newblock \showarticletitle{Semi-supervised classification with graph
  convolutional networks}. In \bibinfo{booktitle}{\emph{International
  Conference on Learning Representations}}.
\newblock


\bibitem[Wang et~al\mbox{.}(2021)]%
        {gan_cv}
\bibfield{author}{\bibinfo{person}{Zhengwei Wang}, \bibinfo{person}{Qi She},
  {and} \bibinfo{person}{Tomas~E Ward}.} \bibinfo{year}{2021}\natexlab{}.
\newblock \showarticletitle{Generative adversarial networks in computer vision:
  A survey and taxonomy}.
\newblock \bibinfo{journal}{\emph{ACM Computing Surveys (CSUR)}}
  \bibinfo{volume}{54}, \bibinfo{number}{2} (\bibinfo{year}{2021}),
  \bibinfo{pages}{1--38}.
\newblock


\bibitem[Williams and Hoel(2003)]%
        {arima}
\bibfield{author}{\bibinfo{person}{Billy~M Williams} {and}
  \bibinfo{person}{Lester~A Hoel}.} \bibinfo{year}{2003}\natexlab{}.
\newblock \showarticletitle{Modeling and forecasting vehicular traffic flow as
  a seasonal ARIMA process: Theoretical basis and empirical results}.
\newblock \bibinfo{journal}{\emph{Journal of transportation engineering}}
  \bibinfo{volume}{129}, \bibinfo{number}{6} (\bibinfo{year}{2003}),
  \bibinfo{pages}{664--672}.
\newblock


\bibitem[Wu et~al\mbox{.}(2004)]%
        {svr}
\bibfield{author}{\bibinfo{person}{Chun-Hsin Wu}, \bibinfo{person}{Jan-Ming
  Ho}, {and} \bibinfo{person}{Der-Tsai Lee}.} \bibinfo{year}{2004}\natexlab{}.
\newblock \showarticletitle{Travel-time prediction with support vector
  regression}.
\newblock \bibinfo{journal}{\emph{IEEE Transactions on Intelligent
  Transportation Systems}} \bibinfo{volume}{5}, \bibinfo{number}{4}
  (\bibinfo{year}{2004}), \bibinfo{pages}{276--281}.
\newblock


\bibitem[Wu et~al\mbox{.}(2020b)]%
        {ast}
\bibfield{author}{\bibinfo{person}{Sifan Wu}, \bibinfo{person}{Xi Xiao},
  \bibinfo{person}{Qianggang Ding}, \bibinfo{person}{Peilin Zhao},
  \bibinfo{person}{Ying Wei}, {and} \bibinfo{person}{Junzhou Huang}.}
  \bibinfo{year}{2020}\natexlab{b}.
\newblock \showarticletitle{Adversarial sparse transformer for time series
  forecasting}.
\newblock \bibinfo{journal}{\emph{Advances in Neural Information Processing
  Systems}} (\bibinfo{year}{2020}), \bibinfo{pages}{17105--17115}.
\newblock


\bibitem[Wu et~al\mbox{.}(2021)]%
        {gnn_survey}
\bibfield{author}{\bibinfo{person}{Zonghan Wu}, \bibinfo{person}{Shirui Pan},
  \bibinfo{person}{Fengwen Chen}, \bibinfo{person}{Guodong Long},
  \bibinfo{person}{Chengqi Zhang}, {and} \bibinfo{person}{Philip~S. Yu}.}
  \bibinfo{year}{2021}\natexlab{}.
\newblock \showarticletitle{A comprehensive survey on graph neural networks}.
\newblock \bibinfo{journal}{\emph{IEEE Transactions on Neural Networks and
  Learning Systems}} \bibinfo{volume}{32}, \bibinfo{number}{1}
  (\bibinfo{year}{2021}), \bibinfo{pages}{4--24}.
\newblock


\bibitem[Wu et~al\mbox{.}(2020a)]%
        {mtgnn}
\bibfield{author}{\bibinfo{person}{Zonghan Wu}, \bibinfo{person}{Shirui Pan},
  \bibinfo{person}{Guodong Long}, \bibinfo{person}{Jing Jiang},
  \bibinfo{person}{Xiaojun Chang}, {and} \bibinfo{person}{Chengqi Zhang}.}
  \bibinfo{year}{2020}\natexlab{a}.
\newblock \showarticletitle{Connecting the dots: multivariate time series
  forecasting with graph neural networks}. In
  \bibinfo{booktitle}{\emph{Proceedings of the ACM SIGKDD International
  Conference on Knowledge Discovery \& Spec- tral temporal graph neural network
  for multivariate time- series forecasting}}. \bibinfo{publisher}{Association
  for Computing Machinery}, \bibinfo{pages}{753–763}.
\newblock


\bibitem[Wu et~al\mbox{.}(2019)]%
        {graph_wavenet}
\bibfield{author}{\bibinfo{person}{Zonghan Wu}, \bibinfo{person}{Shirui Pan},
  \bibinfo{person}{Guodong Long}, \bibinfo{person}{Jing Jiang}, {and}
  \bibinfo{person}{Chengqi Zhang}.} \bibinfo{year}{2019}\natexlab{}.
\newblock \showarticletitle{Graph WaveNet for deep spatial-temporal graph
  modeling}. In \bibinfo{booktitle}{\emph{Proceedings of the International
  Joint Conference on Artificial Intelligence}}. \bibinfo{publisher}{AAAI
  Press}, \bibinfo{pages}{1907--1913}.
\newblock


\bibitem[Yoon et~al\mbox{.}(2019)]%
        {timegan}
\bibfield{author}{\bibinfo{person}{Jinsung Yoon}, \bibinfo{person}{Daniel
  Jarrett}, {and} \bibinfo{person}{Mihaela van~der Schaar}.}
  \bibinfo{year}{2019}\natexlab{}.
\newblock \showarticletitle{Time-series generative adversarial networks}. In
  \bibinfo{booktitle}{\emph{Advances in Neural Information Processing
  Systems}}. \bibinfo{pages}{5508–5518}.
\newblock


\bibitem[Yu et~al\mbox{.}(2018)]%
        {stgcn}
\bibfield{author}{\bibinfo{person}{Bing Yu}, \bibinfo{person}{Haoteng Yin},
  {and} \bibinfo{person}{Zhanxing Zhu}.} \bibinfo{year}{2018}\natexlab{}.
\newblock \showarticletitle{Spatio-temporal graph convolutional networks: A
  deep learning framework for traffic forecasting}. In
  \bibinfo{booktitle}{\emph{Proceedings of the International Joint Conference
  on Artificial Intelligence}}. \bibinfo{publisher}{AAAI Press},
  \bibinfo{address}{Palo Alto, CA USA}, \bibinfo{pages}{3634--3640}.
\newblock


\bibitem[Yu et~al\mbox{.}(2022)]%
        {rgsl}
\bibfield{author}{\bibinfo{person}{Hongyuan Yu}, \bibinfo{person}{Ting Li},
  \bibinfo{person}{Weichen Yu}, \bibinfo{person}{Jianguo Li},
  \bibinfo{person}{Yan Huang}, \bibinfo{person}{Liang Wang}, {and}
  \bibinfo{person}{Alex Liu}.} \bibinfo{year}{2022}\natexlab{}.
\newblock \showarticletitle{Regularized Graph Structure Learning with Semantic
  Knowledge for Multi-variates Time-Series Forecasting}. In
  \bibinfo{booktitle}{\emph{Proceedings of the Thirty-First International Joint
  Conference on Artificial Intelligence}}. \bibinfo{publisher}{AAAI Press},
  \bibinfo{pages}{2362--2368}.
\newblock


\bibitem[Yu et~al\mbox{.}(2017)]%
        {cnn-rnn-2}
\bibfield{author}{\bibinfo{person}{Haiyang Yu}, \bibinfo{person}{Zhihai Wu},
  \bibinfo{person}{Shuqin Wang}, \bibinfo{person}{Yunpeng Wang}, {and}
  \bibinfo{person}{Xiaolei Ma}.} \bibinfo{year}{2017}\natexlab{}.
\newblock \showarticletitle{Spatiotemporal recurrent convolutional networks for
  traffic prediction in transportation networks}.
\newblock \bibinfo{journal}{\emph{Sensors}} \bibinfo{volume}{17},
  \bibinfo{number}{7} (\bibinfo{year}{2017}), \bibinfo{pages}{1501}.
\newblock


\bibitem[Zhang et~al\mbox{.}(2020)]%
        {slcnn}
\bibfield{author}{\bibinfo{person}{Qi Zhang}, \bibinfo{person}{Jianlong Chang},
  \bibinfo{person}{Gaofeng Meng}, \bibinfo{person}{Shiming Xiang}, {and}
  \bibinfo{person}{Chunhong Pan}.} \bibinfo{year}{2020}\natexlab{}.
\newblock \showarticletitle{Spatio-Temporal Graph Structure Learning for
  Traffic Forecasting}.
\newblock \bibinfo{journal}{\emph{Proceedings of the AAAI Conference on
  Artificial Intelligence}} \bibinfo{volume}{34}, \bibinfo{number}{01}
  (\bibinfo{year}{2020}), \bibinfo{pages}{1177--1185}.
\newblock


\bibitem[Zhang et~al\mbox{.}(2021)]%
        {trafficgan}
\bibfield{author}{\bibinfo{person}{Yuxuan Zhang}, \bibinfo{person}{Senzhang
  Wang}, \bibinfo{person}{Bing Chen}, \bibinfo{person}{Jiannong Cao}, {and}
  \bibinfo{person}{Zhiqiu Huang}.} \bibinfo{year}{2021}\natexlab{}.
\newblock \showarticletitle{TrafficGAN: Network-scale deep traffic prediction
  with generative adversarial nets}.
\newblock \bibinfo{journal}{\emph{IEEE Transactions on Intelligent
  Transportation Systems}} \bibinfo{volume}{22}, \bibinfo{number}{1}
  (\bibinfo{year}{2021}), \bibinfo{pages}{219--230}.
\newblock


\bibitem[Zheng et~al\mbox{.}(2020)]%
        {gman}
\bibfield{author}{\bibinfo{person}{Chuanpan Zheng}, \bibinfo{person}{Xiaoliang
  Fan}, \bibinfo{person}{Cheng Wang}, {and} \bibinfo{person}{Jianzhong Qi}.}
  \bibinfo{year}{2020}\natexlab{}.
\newblock \showarticletitle{{GMAN:} {A} graph multi-attention network for
  traffic prediction}. In \bibinfo{booktitle}{\emph{Proceedings of the AAAI
  Conference on Artificial Intelligence}}. \bibinfo{publisher}{AAAI Press},
  \bibinfo{address}{Palo Alto, CA USA}, \bibinfo{pages}{1234--1241}.
\newblock


\bibitem[Zivot and Wang(2006)]%
        {var}
\bibfield{author}{\bibinfo{person}{Eric Zivot} {and} \bibinfo{person}{Jiahui
  Wang}.} \bibinfo{year}{2006}\natexlab{}.
\newblock \showarticletitle{Vector autoregressive models for multivariate time
  series}.
\newblock \bibinfo{journal}{\emph{Modeling financial time series with
  S-PLUS{\textregistered}}} (\bibinfo{year}{2006}), \bibinfo{pages}{369--413}.
\newblock


\end{thebibliography}




\end{document}